\documentclass[11pt]{article}
\usepackage[margin=1.0in]{geometry}

\usepackage{amsmath, amsthm, amssymb}
\usepackage{hyperref}
\hypersetup{pdftex,colorlinks=true,allcolors=blue}
\usepackage{hypcap}

\usepackage{commath}
\usepackage{comment}
\usepackage{chngcntr}
\usepackage{apptools}
\AtAppendix{\counterwithin{lemma}{section}}

\newcommand{\R}{\mathbb{R}}
\newcommand{\E}{\mathbf{E}}
\newcommand{\Var}{{\mathsf{var}}} 
\newcommand{\I}{\mathbf{1}}

\DeclareMathOperator*{\argmax}{arg\,max}

\def\cM{{\mathcal M}}

\def\sA{{\mathsf A}}

\def\sM{{\mathsf M}}

\def\sU{{\mathsf U}}
\def\sV{{\mathsf V}}
\def\sW{{\mathsf W}}
\def\sX{{\mathsf X}}
\def\sY{{\mathsf Y}}
\def\sZ{{\mathsf Z}}

\def\rd{{\rm d}}
\def\PP{{\mathbb P}}

\def\deq{:=}
\def\wh#1{{\widehat{#1}}}
\def\eps{\varepsilon}

\newtheorem{definition}{Definition}
\newtheorem{theorem}{Theorem}
\newtheorem{lemma}{Lemma}
\newtheorem{proposition}{Proposition}
\newtheorem{corollary}{Corollary}

\newcommand\blfootnote[1]{%
	\begingroup
	\renewcommand\thefootnote{}\footnote{#1}%
	\addtocounter{footnote}{-1}%
	\endgroup
}

\title{Minimum Excess Risk in Bayesian Learning}

\begin{document}
	
	\author{Aolin Xu \and Maxim Raginsky}
	\date{}
	\maketitle
	
	\begin{abstract}
		We analyze the best achievable performance of Bayesian learning under generative models by defining and upper-bounding the minimum excess risk (MER): the gap between the minimum expected loss attainable by learning from data and the minimum expected loss that could be achieved if the model realization were known. 
		The definition of MER provides a principled way to define different notions of uncertainties in Bayesian learning, including the aleatoric uncertainty and the minimum epistemic uncertainty.
		Two methods for deriving upper bounds for the MER are presented. The first method, generally suitable for Bayesian learning with a parametric generative model, upper-bounds the MER by the conditional mutual information between the model parameters and the quantity being predicted given the observed data. It allows us to quantify the rate at which the MER decays to zero as more data becomes available. Under realizable models, this method also relates the MER to the richness of the generative function class, notably the VC dimension in binary classification. The second method, particularly suitable for Bayesian learning with a parametric predictive model, relates the MER to the minimum estimation error of the model parameters from data. It explicitly shows how the uncertainty in model parameter estimation translates to the MER and to the final prediction uncertainty.
		We also extend the definition and analysis of MER to the setting with multiple model families and the setting with nonparametric models.
		Along the discussions we draw some comparisons between the MER in Bayesian learning and the excess risk in frequentist learning.
	\end{abstract}
	
	
	\blfootnote{xuaolin@gmail.com, maxim@illinois.edu}
	
	\setcounter{tocdepth}{4} 
	\tableofcontents 
	
	\section{Introduction}
	
	Bayesian learning under generative models has been gaining considerable attention in recent years as an alternative to the frequentist learning.
	In the Bayesian setting, the observed data $Z^n = ((X_1,Y_1),\ldots,(X_n,Y_n))$ is modeled as conditionally i.i.d.\ samples generated from a probabilistic model given the model realization $P_{Z|W}$, while the model is treated either as a random element of some parametric model family drawn according to a prior distribution of the model parameters $W$, or as a nonparametric random process~\cite{BishopBook06}. The task of Bayesian learning is to predict a quantity of interest $Y$ based on the observed data $(X,Z^n)$ and the structure of the model, while the quality of prediction can be assessed by the expected loss with respect to some loss function $\E[\ell(Y,\psi(X,Z^n))]$.
	Although Bayesian learning often relies on posterior sampling or approximation techniques \cite{MCMC_handbook,viBlei17}
	and hence has much higher computational complexity than its frequentist counterpart, 
	the Bayesian viewpoint has many attractive features, e.g., reducing overfitting \cite{Neal_BNN}, quantifying uncertainty in making predictions \cite{Blundell15_w_uncert, Gal2016Dropout,GP_book}, enabling model compression \cite{Bcompres17}, etc. 
	In contrast with the growing attention to the algorithmic side of Bayesian learning, its performance analysis
	is relatively scarce compared to the volume of literature on the theoretical analysis of frequentist learning.
	In this paper, we set aside the computational issues in Bayesian learning and focus on analyzing its best achievable performance under the generative model with respect to general loss functions.
	
	\subsection{Overview of the presentation}
	In Section~\ref{sec:formulation_single}, we define the \textit{minimum excess risk} (MER) in Bayesian learning as the gap between the Bayes risk $R_\ell(Y|X,Z^n)$, defined as the minimum expected loss attainable by learning from the data, and its fundamental limit $R_\ell(Y|X,W)$, defined as the minimum expected loss that would be achieved if the model parameters were known. 
	The MER is an algorithm-independent quantity that captures the uncertainty arising from the lack of knowledge of the underlying model parameters, commonly known as \textit{epistemic uncertainty}. Its value and rate of convergence to zero reflect the difficulty of the learning problem. 
	The decomposition of the Bayes risk into its fundamental limit and MER also provides rigorous definitions of the aleatoric uncertainty and the minimum epistemic uncertainty in Bayesian learning, the quantification of which has become an important research topic in recent years \cite{NIPS2017_7141,Depeweg2018Decomp,Hllermeier2019Aleatoric}.
	To the authors' knowledge, the general definition of MER is new, and has not been systematically studied before.
	
	We then present two methodologies for deriving upper bounds on the MER. We first show in Section~\ref{sec:ub_mi} that, under a generic parametric generative model and for a wide range of loss functions, the MER can be upper-bounded in terms of the conditional mutual information between the model parameters and the quantity being predicted given the observed data, $I(W;Y|X,Z^n)$. This leads to asymptotic upper bounds on the MER that scale as $O({d}/{n})$ or $O({\sqrt{{d}/{n}}})$ depending on the loss function, where $d$ is the dimension of the parameter space and $n$ is the data size.
	It also reveals an MER-information relationship in Bayesian learning, echoing the generalization-information relationship in frequentist learning \cite{XuRaginsky17}. Under realizable models, it is shown that for any bounded loss function, the MER for binary classification scales as $O(d/n)$, where $d$ is the VC dimension of the generative function class.
	Next, in Section~\ref{sec:MER_W_est_error}, we explore two alternative methods for bounding the MER. One relies on the smoothness of the decision rule in the model parameters, while the other relies on the smoothness of the minimum expected loss as a functional of the predictive model.
	The resulting upper bounds single out the dependence of the MER on the minimum achievable estimation error of the model parameters from the data, e.g., on the minimum mean square error (MMSE) of the estimated model parameters $R_2(W|X,Z^n)$.
	It explicitly shows how the difficulty of model parameter estimation translates into the difficulty of prediction due to the model uncertainty.
	
	The analysis of the MER in the single model family setting can be extended to the setting with multiple model families.
	The definition of MER can also be extended to the setting with nonparametric generative models, such as Gaussian processes, and the analysis of MER based on conditional mutual information carries over to this setting. These extensions are briefly discussed in Section~\ref{sec:extensions}.
	We close by summarizing the results and making some comparisons between the MER in Bayesian learning and the excess risk in frequentist learning in Section~\ref{sec:discuss}.

	\subsection{Relation to existing works}
	\subsubsection{Accumulated excess risk for log loss}
	The closest connection between this work and prior literature is the MER for the logarithmic (log) loss defined in this paper and the \textit{accumulated} excess risks for the log loss defined in Bayesian universal source coding \cite{Davisson1973}, Bayesian sequential prediction \cite{Merhav98}, Bayesian density estimation \cite{haussler97}, and Bayesian supervised learning \cite{Baxter1997ABT}, all of which turn out to be the mutual information between the model parameters and the observed data, and are achieved by the posterior predictive distribution as a soft predictor.
	The only work where more general loss functions is considered is the study of sequential prediction in \cite{Merhav98}, where an upper bound on the accumulated excess risk for bounded loss functions is derived.
	Our definition of MER goes beyond the log loss to general loss functions, which can be unbounded, and the MER in general is achieved not necessarily by the posterior predictive distribution, but by some hard predictor according to the loss function.
	In Section~\ref{sec:ub_mi} we show that the MER for the log loss is nevertheless an important quantity, as it can be used to upper-bound the MER for many other loss functions.
	Most of the above works, with the exception of \cite{Baxter1997ABT}, considered only unsupervised learning, while our results hold for both supervised and unsupervised learning.
	In addition, the MER defined in this work is the \textit{instantaneous} excess risk, instead of the accumulated risk studied in above works, thus is amenable to more refined analyses.
	Another closely related work that considered both supervised learning and instantaneous risk is \cite{HausslerVC94}, where the Bayes risk of binary classification with the zero-one loss is derived by relating it to the accumulated log loss, and is further related to the VC dimension of the generative function class. As only realizable models are considered in \cite{HausslerVC94}, the Bayes risk there is equal to the MER. In Section~\ref{sec:MER_rlz}, we also study the MER under realizable models, but our results go beyond binary classification and the zero-one loss.

	\subsubsection{Convergence of posterior distribution}
	A classical frequentist analysis of Bayesian inference is the convergence of the \emph{posterior parameter distribution} to the true model parameters, assuming the data is sampled from some \emph{fixed} model with the true parameters \cite{LeCamYang2000,ghosal2000,ghosal2007}. This analysis has recently been extended to deep neural network models \cite{NIPS2018_7372}.
	The convergence of the \emph{posterior predictive distribution} has also been studied under the same assumption \cite{Barron98}.
	The main difference between these works and ours is the assumption on the data distribution. In our work, the underlying data distribution considered in the performance analysis stays the same as the generative model based on which the optimal predictor, or the learning algorithm, is derived. In other words, the model parameters are assumed to be randomly drawn from the prior, and the data samples are drawn from the model given the model parameters.
	In addition, rather than the convergence of the posterior of the model parameters, we are interested in the accuracy of the predicted quantity of interest.
	In Section~\ref{sec:MER_W_est_error} we reveal how this accuracy explicitly depends on the accuracy of the model parameter estimation, by studying the expected deviation of the \textit{posterior predictive distribution} from the random true model.

	\subsubsection{PAC-Bayes}
	Another loosely related line of work in statistical learning is the PAC-Bayes framework in the frequentist setting \cite{DM_PACB03, STW_pac_Bayes97, Zhang_it_est06}
	and its extension as the Bayes mixture model \cite{BayeMixture03}.
	The main difference between the Bayesian setting considered here and the PAC-Bayes framework is again the underlying data distribution. For the former, the data distribution is restricted to a parametric or nonparametric family of generative models with the data samples being \emph{conditionally} i.i.d.\ given the model realization, and there is virtually no restriction on candidate predictors. For the latter, the data samples are drawn \textit{unconditionally} i.i.d.\ according to a completely unknown distribution, and the hypothesized Bayes-like update takes place in a \textit{hypothesis space} consisting of admissible predictors only. The excess risk studied in this paper is thus not directly related to the generalization error or excess risk in the PAC-Bayes method.
	Nevertheless, the MER-information relationship in Theorem~\ref{thm:regression_genloss_MI} is an interesting analogue of the generalization-information relationship in the frequentist setting \cite[Theorem~1]{XuRaginsky17} that leads to an information-theoretic derivation of the PAC-Bayes algorithm.
	
	\subsection{A note on notation}
	Throughout the paper, random variables are denoted by uppercase letters and their realizations are in the corresponding lowercase letters. To keep the notation uncluttered, we may use $K_{U|v}$, $P_{U|v}$ and $\E[U|v]$ respectively to denote the probability transition kernel $K_{U|V=v}$, the conditional distribution $P_{U|V=v}$ and the conditional expectation $\E[U|V=v]$. 
	When the conditioning variables are written in uppercase letters, these quantities are random, and expectations can be taken with respect to the conditioning variables.
	Throughout the paper, $D(\cdot,\cdot)$ denotes a generic statistical distance, while the KL divergence is denoted by $D_{\rm KL}(\cdot \| \cdot)$.
	All probability spaces considered in this paper are Borel spaces, and all functions are measurable functions.
	We use natural logarithms throughout the paper.

	\section{Model and definitions}\label{sec:formulation_single}
	\subsection{Bayesian learning under parametric generative model}
	The basic task in supervised learning is to construct an accurate predictor of $Y$, a quantify of interest, given an observation $X$, where the knowledge of the joint distribution of $X$ and $Y$ is vague but can be inferred from a historical dataset $((X_1, Y_1),\ldots,(X_n, Y_n))$.	
	In the \emph{model-based learning} framework, a.k.a.\ learning under a \emph{generative model}, the joint distribution of $X$ and $Y$ is assumed to be an element of a known model family.
	The model family can be either parametric or nonparametric. 
	We focus on the parametric case in this work, and defer a brief discussion on the nonparametric case to Section~\ref{sec:nonpara}.
	In the case of parametric modeling, the model family is a collection of parametrized distributions ${\mathcal M}=\{P_{X,Y| w}, w\in\sW\}$, where $w$ represents the vector of unknown model parameters belonging to some space $\sW$.
	Under the \emph{Bayesian formulation}, the vector of model parameters $W$ is itself treated as a random quantity with a prior distribution $P_W$, while the data samples are conditionally i.i.d.\ given $W$. 
	Formally, the model parameters $W$, the data samples $Z^n \deq (Z_1,\ldots,Z_n)$ with $Z_i \deq (X_i,Y_i)$, $i =1,\ldots,n$, and the pair $Z = (X,Y)$ consisting of the fresh observation $X$ and the quantity $Y$ to be predicted are assumed to be generated according to the joint distribution
	\begin{align}\label{eq:joint_dist}
		P_{W, Z^n, Z} = P_W  \Big(\prod\limits_{i=1}^n P_{Z_i|W}\Big)  P_{Z|W},
	\end{align}
	where $P_{Z_i|W} = P_{Z|W}$ for each $i$.
	As an example of the above model, the \emph{predictive modeling} framework, a.k.a.\ {probabilistic discriminative model} \cite{BishopBook06}, further assumes that $P_{Z|W}$ factors as $P_{Z|W} = P_{X|W} K_{Y|X,W}$, with some probability transition kernel $K_{Y|X,W}$ directly describing the true predictive distribution of the quantity of interest given the observation and model parameters. It is often further assumed that $X$ is independent of $W$ under the predictive modeling framework.
	Note that the above models and the following definitions encompass the unsupervised learning problem as well, where one just ignores the observations $(X^n,X)$ so that $Z_i=Y_i$ and $Z=Y$.
	
	Under the generative model \eqref{eq:joint_dist}, the \textit{Bayesian learning} problem can be phrased as a Bayes decision problem of predicting $Y$ based on $X$ and the labeled observations $Z^n$. 
	Given an action space $\sA$ and a loss function $\ell : \sY \times \sA \rightarrow \R$, 	
	a {\em decision rule} $\psi: \sX \times \sZ^n \rightarrow \sA$ that maps observations to an action is sought to make the expected loss $\E[\ell(Y,\psi(X,Z^n))]$ small.
	A decision rule that minimizes the expected loss among all decision rules is called a \emph{Bayes decision rule}.
	The corresponding minimum expected loss is defined as the \emph{Bayes risk} in Bayesian learning:
	\begin{definition}
		In Bayesian learning, the Bayes risk with respect to a loss function $\ell$ is defined as
		\begin{align}\label{eq:RB_def}
			R_\ell(Y|X,Z^n) \deq \inf_{\psi: \sX \times \sZ^n \rightarrow \sA} \E[\ell(Y, \psi(X,Z^n))] ,
		\end{align}
		where the infimum is taken over all decision rules such that the above expectation is defined.
	\end{definition}
	
	
	\subsection{A data processing inequality for Bayes risk}
	To better understand the definition of $R_\ell(Y|X,Z^n)$, we give a brief review of the general definition of the Bayes risk and prove a useful property of it.
	Given a random element $Y$ of $\sY$, the quantity
	\begin{align}\label{eq:def_Brisk}
		R_\ell(Y) \deq \inf_{a \in \sA} \E[\ell(Y,a)]
	\end{align}
	is known as the \textit{Bayes envelope} \cite{Merhav98} or the \textit{generalized entropy} \cite{gunwald2004} of $Y$. Given a random element $V$ of some space $\sV$ jointly distributed with $Y$, the general definition of the \emph{Bayes risk}
	\begin{align}\label{eq:def_cond_Brisk}
		R_\ell(Y|V)\deq \inf_{\psi: \sV \rightarrow \sA}\E[\ell(Y,\psi(V))]
	\end{align}
	is the minimum expected loss of predicting $Y$ given $V$.
	It can be expressed as
	the expectation of the \textit{conditional Bayes envelope}
	$
	R_\ell(Y|V=v) \deq \inf_{a \in \sA} \E[\ell(Y,a)|V=v]
	$
	with respect to $V$, as
	$
	R_\ell(Y|V) = \int_\sV P_V(\dif v) R_\ell(Y|V=v) .
	$
	The Bayes risk $R_\ell(Y|V)$ can thus be viewed as a \textit{generalized conditional entropy} of $Y$ given $V$ \cite{DeGroot62,FarniaTse16}. 
	The following lemma states that the Bayes risk satisfies a \emph{data processing inequality}.
	\begin{lemma}\label{lm:DPI_RB_gen}
		Suppose the random variables $U$, $V$ and $Y$ form a Markov chain $U-V-Y$; in other words, $Y$ and $U$ are conditionally independent given $V$. Then, for any loss function $\ell$, the Bayes risk of predicting $Y$ from $U$ is at least as large as the Bayes risk of predicting $Y$ from $V$, i.e.,
		\begin{align}
			R_{\ell}(Y|U) \ge R_{\ell}(Y|V) .
		\end{align}
	\end{lemma}
	\begin{proof}
		Let $\psi$ be a Bayes decision rule for predicting $Y$ from $U$. Upon observing $V$, a random variable $U'$ can be sampled from $P_{U|V}$, conditionally independent of $(U,Y)$ given $V$. Then $\psi(U')$ serves as a randomized prediction of $Y$ from $V$.
			As all probability spaces under consideration are Borel spaces, the sampling of $U'$ conditional on $V$ can be realized by a function $f:\sV\times [0,1] \rightarrow \sU$ of $V$ and an independent random variable $T$ uniformly distributed on $[0,1]$, such that $P_{f(V,T)|V} = P_{U|V}$ \cite[Lemma~3.22]{Kallenberg}.
			We have
			\begin{align}
				R_\ell(Y|V) 
				&\le \inf_{t\in[0,1]} \E[\ell(Y, \psi(f(V,t)))] \label{eq:pf_Rexcess_nneg_6} \\
				&\le \E[\ell(Y, \psi(f(V,T)))] \label{eq:pf_Rexcess_nneg_0} \\
				&= \E[\ell(Y, \psi(U'))] \label{eq:pf_Rexcess_nneg_1} \\
				&= \E[\ell(Y, \psi(U))] \label{eq:pf_Rexcess_nneg_2} \\
				&= R_\ell(Y|U) \label{eq:pf_Rexcess_nneg_3}
			\end{align}
			where \eqref{eq:pf_Rexcess_nneg_6} is due to the definition of $R_\ell(Y|V)$ and the fact that $\psi(f(\cdot,t))$ is a map from $\sV$ to $\sY$ for each $t\in[0,1]$; \eqref{eq:pf_Rexcess_nneg_0} follows from the independence between $T$ and $(V,Y)$;
			\eqref{eq:pf_Rexcess_nneg_2} follows from the fact that $P_{U'|V} = P_{U|V,Y}$ due to the Markov chain $U-V-Y$, hence $P_{U',V,Y} = P_{U,V,Y}$; and \eqref{eq:pf_Rexcess_nneg_3} follows from the definition of $\psi$. 
	\end{proof}
	In view of the definition of $\mu$-entropy in \eqref{eq:def_mu_ent} and \eqref{eq:def_cond_mu_ent} in Section~\ref{sec:MI_log}, the classic data processing inequality for mutual information stating that $I(U;Y)\le I(V;Y)$ in a Markov chain $U-V-Y$~\cite{Cover_book} can be derived from Lemma~\ref{lm:DPI_RB_gen} applied to the log loss.
	More importantly, Lemma~\ref{lm:DPI_RB_gen} extends the \emph{value of information principle} in Bayes decision making~\cite{DeGroot62}, which states that $R_\ell(Y)\ge R_\ell(Y|V)$, as it can be viewed as a special case of Lemma~\ref{lm:DPI_RB_gen} when $U$ is independent of $(V,Y)$. Lemma~\ref{lm:DPI_RB_gen} also extends the \emph{principle of total evidence} ~\cite{Good_pte}, a.k.a.\ the \emph{value of knowledge theorem}~\cite{gen_learn_Huttegger}, which states that $R_\ell(Y|V_1)\ge R_\ell(Y|V_1,V_2)$ for arbitrary random variables $V_1$ and $V_2$ jointly distributed with $Y$. While the original argument in \cite{Good_pte} overlooked the randomness of $V_1$, this principle can be rigorously justified by Lemma~\ref{lm:DPI_RB_gen} as $V_1-(V_1,V_2)-Y$ always form a Markov chain.
	It is also apparent from Lemma~\ref{lm:DPI_RB_gen} or its proof that randomizing the decision rule does not help to decrease the expected loss in Bayes decision making, as $(T,V) - V - Y$ form a Markov chain for any independent random variable $T$ to be used in the randomized decision rule.

	\subsection{Definition of minimum excess risk}
	An immediate consequence of Lemma~\ref{lm:DPI_RB_gen} in Bayesian learning is that the Bayes risk $R_\ell(Y|X,Z^n)$ decreases as the data size $n$ increases, as $(X,Z^n) - (X,Z^{n+1}) - Y$ form a Markov chain. A special case of this result for linear regression with quadratic loss appears in \cite{Qazaz1997}. While $R_\ell(Y|X,Z^n)$ decreases in $n$, it will not necessarily vanish as $n\rightarrow\infty$. We define the \emph{fundamental limit of the Bayes risk} as the minimum expected loss when the model parameters $W$ are known, which is attained by some ``omniscient'' decision rule $\Psi: \sX\times\sW \rightarrow \sA$ that can directly access the model parameters.
	\begin{definition}
		In Bayesian learning, the fundamental limit of the Bayes risk with respect to a loss function $\ell$ is defined as
		\begin{align}\label{eq:RBW_def}
			R_\ell(Y|X,W) = \inf_{\Psi: \sX\times\sW \rightarrow \sA} \E[\ell(Y,\Psi(X,W))] .
		\end{align}
	\end{definition}
	For any feasible decision rule $\psi: \sX \times \sZ^n \rightarrow \sA$, we can define its \emph{excess risk} as the gap between its expected loss $\E[\ell(Y, \psi(X,Z^n))]$ and $R_\ell(Y|X,W)$.
	In this work, our interest is in the gap between the Bayes risk $R_\ell(Y|X,Z^n)$ and its fundamental limit $R_\ell(Y|X,W)$, which is the minimum achievable excess risk among all feasible decision rules:
	\begin{definition}
		The {minimum excess risk} (MER) with respect to a loss function $\ell$ is defined as
		\begin{align}\label{eq:MER_def}
			{\rm MER}_{\ell} \deq R_\ell(Y|X,Z^n) - R_\ell(Y|X,W) .
		\end{align}
	\end{definition}
	The MER defined above is an algorithm-independent quantity. It quantifies the regret of the best decision rule that has access to data, but not to model parameters, relative to the best ``omniscient" decision rule. It thus reflects the difficulty of the learning problem, which comes from the lack of knownedge of $W$.
	This is better illustrated by decomposing the Bayes risk as
	\begin{align}
		R_\ell(Y|X,Z^n) =  R_\ell(Y|X,W) + {\rm MER}_\ell .
	\end{align}
	If we view the Bayes risk as a measure of the minimum prediction uncertainty, this decomposition allows us to give formal definitions of the ``aleatoric'' uncertainty and the minimum ``epistemic'' uncertainty \cite{NIPS2017_7141}.
	The first term, the fundamental limit of the Bayes risk, can be viewed as the aleatoric part of the minimum prediction uncertainty, which exists even when the model parameters are known.
	The second term, the MER, can be viewed as the epistemic part of the minimum prediction uncertainty, which is due to the lack of knowledge of $W$.
	In \cite{Depeweg2018Decomp}, a decomposition of uncertainty is proposed for the log loss and the quadratic loss, where the epistemic uncertainty is defined as $R_\ell(Y|X)-R_\ell(Y|X,W)$ when expressed by our notation; however, this definition does not take the observed data into consideration, thus does not reflect the intuitive expectation that the epistemic uncertainty should decrease as the data size increases \cite{Hllermeier2019Aleatoric}.
	On the contrary, defining the minimum epistemic uncertainty as the MER has the advantage that it becomes smaller as more data is observed, as asserted by the following result.
	\begin{theorem}\label{prop:DPI_RB}
		For any loss function, ${\rm MER}_\ell$ decreases in the data size $n$, and ${\rm MER}_\ell\ge 0$ for all $n$.
	\end{theorem}
	\begin{proof}
		The claim that ${\rm MER}_\ell$ decreases in $n$ is due to the previously justified fact that $R_\ell(Y|X,Z^n)$ decreases in $n$ as a consequence of Lemma~\ref{lm:DPI_RB_gen}. The claim that ${\rm MER}_\ell\ge 0$ is due to the Markov chain $(X,Z^n)-(X,W)-Y$ and Lemma~\ref{lm:DPI_RB_gen}.
	\end{proof}
	Intuitively, we expect that ${\rm MER}_\ell \downarrow 0$ as $n\rightarrow\infty$. However, except for the log loss, there are few results in the literature regarding this convergence in the general case, or regarding how the MER depends on the estimation error of the model parameters.
	In the following two sections, we use different methods to derive upper bounds on MER for general loss functions. We show that, in many cases, the MER can be upper-bounded either in terms of the conditional mutual information $I(W;Y|X,Z^n)$, or in terms of the minimum achievable estimation error of $W$ from $(X,Z^n)$.
	These results reflect how the MER depends on the joint distribution in \eqref{eq:joint_dist}, in particular on $P_{Z|W}$ and $P_W$, as well as on the loss function and the data size.

	\section{Upper bounds via conditional mutual information}\label{sec:ub_mi}
	The first method for upper-bounding the MER is to relate it to the conditional mutual information $I(W ; Y| X, Z^n)$, which can be further bounded by $\frac{1}{n}I(W;Y^n|X^n)$ or $\frac{1}{n}I(W;Z^n)$. In many cases, it can be shown that the mutual information $I(W;Z^n)$ is sublinear in $n$ \cite{Rissanen84,Clarke_Barron,Cla_Bar94}, which implies that the MER converges to zero as $n \to \infty$.

	\subsection{Logarithmic loss}\label{sec:MI_log}
	We first consider the setting where one makes ``soft'' predictions, such that the action space is the collection of all probability densities $q$ with respect to a common $\sigma$-finite positive measure $\mu$ on $\sY$. The log loss $\ell(y,q) \deq -\log q(y)$ penalizes those densities that assign small probabilities to the outcome $y$.
	Based on the definitions in \eqref{eq:def_Brisk} and \eqref{eq:def_cond_Brisk}, it can be shown that
	\begin{align}\label{eq:def_mu_ent}
		R_{\log}(Y) = H_\mu(Y) \deq - \int_\sY p_Y(y)\log p_Y(y)\mu(\dif y) 
	\end{align}
	and
	\begin{align}\label{eq:def_cond_mu_ent}
		R_{\log}(Y|V) = H_\mu(Y|V) \deq -\int_\sV P_V(\dif v)\int_\sY p_{Y|v}(y) \log p_{Y|v}(y)\mu(\dif y) ,
	\end{align}
	which can be viewed as the \emph{$\mu$-entropy} of $Y$ and the \emph{conditional $\mu$-entropy} of $Y$ given $V$, and the optimal actions are the unconditional density $p_Y$ and the conditional density $p_{Y|v}$ with respect to $\mu$, respectively.
	For instance, if $\sY$ is discrete and $\mu$ is the counting measure, then $R_{\log}(Y) = H(Y)$ and $R_{\log}(Y|V) = H(Y|V)$ are the Shannon and the conditional Shannon entropy; while if $\sY = \R^p$ and $\mu$ is the Lebesgue measure, then $R_{\log}(Y) = h(Y)$ and $R_{\log}(Y|V) = h(Y|V)$ are the differential and the conditional differential entropy. (See \cite{Cover_book} for further background on information theory.)
	With these definitions, the MER for the log loss is the difference between two $\mu$-entropy terms:
	\begin{align}
		{\rm MER}_{\log} &=  H_\mu(Y|X,Z^n) - H_\mu(Y|X,W) . \label{eq:class_log_H}
	\end{align}
	A key observation is that ${\rm MER}_{\log}$ can be expressed in terms of the conditional mutual information:
	\begin{lemma}\label{lm:MER_log_mi}
		For the log loss,
		\begin{align}
			{\rm MER}_{\log} = I(W ; Y| X, Z^n)  \label{eq:log_cond_mi} .
		\end{align}
	\end{lemma}
	\begin{proof}
		The claim follows from the fact that $I(W ; Y| X, Z^n) = H_\mu(Y|X,Z^n) - H_\mu(Y|X,W,Z^n)$ and that $H_\mu(Y|X,W,Z^n) = H_\mu(Y|X,W)$. The second fact is due to the Markov chain $(X,W,Z^n)-(X,W)-Y$ encoded in \eqref{eq:joint_dist} and the definition of the conditional $\mu$-entropy.
	\end{proof}
	Equation~\eqref{eq:log_cond_mi} states that ${\rm MER}_{\log}$ is the average reduction of the uncertainty about $Y$ that comes from the knowledge of $W$, given that $(X,Z^n)$ is already known.
	With this representation, using the conditional independence structure in \eqref{eq:joint_dist} and the data processing inequality in Lemma~\ref{lm:DPI_RB_gen} applied to the $\mu$-entropy, we have:

	\begin{theorem}\label{th:MER_log_I(W;Y^n|X^n)}
		The MER with respect to the log loss can be upper-bounded as
		\begin{align}\label{eq:MER_log_I(W;Y^n|X^n)}
			I(W;Y|X,Z^n) \le \frac{1}{n} I(W;Y^n|X^n) .
		\end{align}
	\end{theorem} 
	\begin{proof}
		For $i=1,\ldots,n-1$, we have
		\begin{align}
			I(W;Y_i|X^n, Y^{i-1}) &= H_\mu(Y_i|X^n, Y^{i-1}) - H_\mu(Y_i|W, X^n, Y^{i-1}) \label{eq:pf_mi_cond_1} \\ 
			&= H_\mu(Y_{i+1}|X^n, Y^{i-1}) - H_\mu(Y_{i+1}|W, X^n, Y^{i-1})  \label{eq:pf_mi_cond_2} \\
			&\ge H_\mu(Y_{i+1}|X^n, Y^{i}) - H_\mu(Y_{i+1}|W, X^n, Y^i)  \label{eq:pf_mi_cond_3} \\
			&= I(W;Y_{i+1}|X^n, Y^{i}) \label{eq:pf_mi_cond_4} 
		\end{align}
		where \eqref{eq:pf_mi_cond_1} is due to the definitions of the conditional mutual information and the conditional $\mu$-entropy in \eqref{eq:def_cond_mu_ent}; \eqref{eq:pf_mi_cond_2} follows from the fact that $(W, X^n, Y^{i-1}, Y_i) \stackrel{\rm d.}{=} (W, X^n, Y^{i-1}, Y_{i+1})$\footnote{For random variables $U$ and $V$, $U\stackrel{\rm d.}{=}V$ means that $U$ and $V$ have the same distribution.}; and \eqref{eq:pf_mi_cond_3} follows from the fact that $H_\mu(Y_{i+1}|X^n, Y^{i-1}) \ge H_\mu(Y_{i+1}|X^n, Y^{i})$ due to Lemma~\ref{lm:DPI_RB_gen}, and the fact that $H_\mu(Y_{i+1}|W, X^n, Y^{i-1}) = H_\mu(Y_{i+1}|W, X^n, Y^{i}) = H_\mu(Y_{i+1}|W,X_{i+1})$ as $Y_{i+1}$ is conditionally independent of everything else given $(W,X_{i+1})$.
		
		Then, from the chain rule of mutual information,
		\begin{align}
			I(W;Y^n|X^n) &= \sum_{i=1}^{n} I(W;Y_i|X^n,Y^{i-1}) \label{eq:pf_mi_cond_5} \\
			&\ge n I(W;Y_{n} | X^{n}, Y^{n-1}) \label{eq:pf_mi_cond_6} \\
			&= n I(W;Y | X, Z^{n-1}) \label{eq:pf_mi_cond_7} \\
			&= n \big(H_\mu(Y|X,Z^{n-1}) - H_\mu(Y|W,X,Z^{n-1}) \big) \label{eq:pf_mi_cond_8} \\
			&\ge n \big(H_\mu(Y|X,Z^{n}) - H_\mu(Y|W,X,Z^{n}) \big) \label{eq:pf_mi_cond_9} \\
			&= n I(W;Y | X, Z^{n}) \label{eq:pf_mi_cond_10} 
		\end{align}
		where \eqref{eq:pf_mi_cond_6} is obtained by repeated application of \eqref{eq:pf_mi_cond_4}; \eqref{eq:pf_mi_cond_7} is due to the fact that  $(W, Z^{n-1}, Z_n) \stackrel{\rm d.}{=} (W, Z^{n-1}, Z)$; and \eqref{eq:pf_mi_cond_9} follows from Lemma~\ref{lm:DPI_RB_gen} and the fact that $Y$ is conditionally independent of everything else given $(W,X)$.
		The claim follows from \eqref{eq:pf_mi_cond_10}.
	\end{proof}

Theorem~\ref{th:MER_log_I(W;Y^n|X^n)} can be weakened to the following corollary using the fact that $I(W;Y^n|X^n) = I(W;Z^n) - I(W;X^n)$. There is no slack when $X$ is independent of $W$.
\begin{corollary}\label{co:classification_log_mi_total}  
	The MER with respect to the log loss can be upper-bounded as
	\begin{align}\label{eq:MER_ub_mi_total}
		I(W;Y|X,Z^n) &\le \frac{1}{n} I(W; Z^n) . 
	\end{align}
\end{corollary}

	Upon maximizing over $P_W$ on both side of \eqref{eq:MER_ub_mi_total}, Corollary~\ref{co:classification_log_mi_total} is reminiscent of the \emph{redundancy-capacity theorem} in universal source coding in the Bayesian setting \cite{Davisson1973,Merhav98}, where the quantity of interest is the minimum \emph{overall} redundancy
	$
	\min_{Q}\E_{P_{W,Z^n}}[-\log{Q(Z^n)} + \log{P_{Z^n|W}(Z^n|W)}]
	$,
	which can be shown to be $I(W;Z^n)$.
	Therefore, from the source coding point of view, ${\rm MER}_{\log}$ in \eqref{eq:log_cond_mi} may be interpreted as the minimum \emph{instantaneous} redundancy of encoding a fresh sample when $n$ data samples are observed, which is shown to be smaller than the normalized minimum overall redundancy by Corollary~\ref{co:classification_log_mi_total}.
	More generally, the mutual information $I(W;Z^n)$ is also known to be the minimum \emph{accumulated} excess risk for the log loss in Bayesian sequential prediction \cite{Merhav98}, Bayesian density estimation \cite{haussler97}, and Bayesian supervised learning \cite{Baxter1997ABT}.
	The non-asymptotic relationships between the instantaneous ${\rm MER}_{\log}$ and the accumulated excess risks shown in Theorem~\ref{th:MER_log_I(W;Y^n|X^n)} and Corollary~\ref{co:classification_log_mi_total} hold for general model $P_{Z|W}$ and prior $P_W$, and allow us to quantify the rate of convergence of ${\rm MER}_{\log}$ by upper-bounding $I(W;Y^n|X^n)$ or $I(W;Z^n)$.
	
	From the results of \cite{Rissanen84,Clarke_Barron,Cla_Bar94}, if $\sW$ is a $d$-dimensional compact subset of $\R^d$ and the model $P_{Z|w}$ is sufficiently smooth in $w$ (see Section~\ref{appd:cond_I(W;Z^n)} for rigorous statements of these conditions), then
	\begin{align}
		I(W;Z^n) = \frac{d}{2}\log \frac{n}{2\pi e} + h(W)
		+ \frac{1}{2}\E\big[\log \det J_{Z|W}\big] + o(1) \quad \text{ as $n\rightarrow\infty$,} \label{eq:mi_expan}
	\end{align}
	where $h(W)$ is the differential entropy of $W$, and, as a functional of $P_{Z|w}$, $J_{Z|w}$ is the Fisher information matrix about $w$ contained in $Z$ with respect to $P_{Z|w}$, and the expectation is taken with respect to $P_W$.
	Due to the logarithmic dependence on $n$ in \eqref{eq:mi_expan} and the chain rule of mutual information, it can be shown that the instantaneous mutual information under the same conditions satisfies $I(W;Z|Z^n) = O(d/n)$ as $n\rightarrow\infty$. 
	This gives us a refined asymptotic upper bound on ${\rm MER}_{\log}$ whenever \eqref{eq:mi_expan} holds than directly applying \eqref{eq:mi_expan} to Corollary~\ref{co:classification_log_mi_total}:
	\begin{theorem}\label{th:MER_log_d/n}
		Under the regularity conditions listed in Section~\ref{appd:cond_I(W;Z^n)} under which \eqref{eq:mi_expan} holds, we have
		\begin{align}
			{\rm MER}_{\log} 
			&= O\Big(\frac{d}{2n}\Big) \quad\text{as $n\rightarrow\infty$.}
		\end{align}
	\end{theorem}
	\begin{proof}
		The proof relies on \cite[Lemma~6]{Haussler_Opper_MI95} which is stated as Lemma~\ref{lm:sum_log_n} in Appendix~\ref{appd:lm_sum_log_n}.
		Suppose $(a_1,a_2,\ldots)$ and $(b_1,b_2,\ldots)$ are two sequences of real numbers such that $a_n = \sum_{i=1}^n b_i$ for all $n$. Lemma~\ref{lm:sum_log_n} states that, if $\lim_{n\rightarrow\infty}{a_n}/{\log n}$ and $\lim_{n\rightarrow\infty}n b_n$ exist, then they are equal. With this result and the chain rule of mutual information, we know that whenever \eqref{eq:mi_expan} holds,
		\begin{align}
	\lim_{n\rightarrow\infty} (n+1) I(W;Z|Z^n) = \lim_{n\rightarrow\infty}\frac{I(W;Z^n)}{\log n}= \frac{d}{2} .
		\end{align}
	The claim follows from the fact that $I(W;Y|X,Z^n) \le I(W;Z|Z^n)$. 
	\end{proof}
	
	As we show next, the representation of ${\rm MER}_{\log}$ via the conditional mutual information in \eqref{eq:log_cond_mi} and the resulting upper bounds derived in this subsection can be used to obtain upper bounds on the MER for other loss functions as well.

	\subsection{Quadratic loss}
	While the log loss is naturally used for assessing ``soft'' predictions, it is also a common practice to make ``hard'' predictions, e.g., the actions can be elements in $\sY$.
	When $\sY = \sA = \R$, a commonly used loss function is the quadratic loss $\ell(y,a) = (y-a)^2$. For any $V$ that statistically depends on $Y$, the conditional Bayes envelope with respect to the quadratic loss is $R_2(Y|V=v) = \Var[Y|v]$, the optimal action is the conditional mean $\E[Y|v]$, and the corresponding Bayes risk 
	\begin{align}
		R_2(Y|V) = \E[\Var[Y|V]]
	\end{align}
	is the minimum mean square error (MMSE) of estimating $Y$ from $V$.
	In this case, the MER in Bayesian learning turns out to be
	\begin{align}\label{eq:R2_Var}
		{\rm MER}_{2} &=  \E\big[\Var[Y|X,Z^n]\big] - \E\big[\Var[Y|X,W]\big] .
	\end{align}
	More generally, when $\sY = \sA = \R^p$ and $\ell(y,a) = \|y-a\|^2$ with $\|\cdot\|$ denoting the $l_2$ norm, the MER in this case is
	\begin{align}
		{\rm MER}_{2} 
		&=  \E\big[\|Y - \E[Y|X,Z^n]\|^2\big] - \E\big[\|Y-\E[Y|X,W]\|^2\big] \\
		&= \E\big[\| \E[Y|X,Z^n] - \E[Y|X,W] \|^2\big] ,
	\end{align}
	where the second equality follows from the fact that $\E[Y|X,W] = \E[Y|X,W,Z^n]$ and the orthogonality principle in MMSE estimation \cite{BH_RPbook}. 
	
	Under the assumption that $\|Y\|\le b$, using a result that connects MMSE difference to conditional mutual information \cite[Theorem~10]{WuVerdu12_fmmseMI}, we can upper-bound ${\rm MER}_{2}$ in terms of $I(W;Y|X,Z^n)$:
	\begin{theorem}\label{prop:regression_quadratic_MI}
		If $\sY = \{y\in\R^p: \|y\|\le b\}$ for some $b>0$, then for the quadratic loss,
		\begin{align}
			{\rm MER}_{2} 
			\le {2b^2} I(W;Y|X,Z^n) . \label{eq:R2_mi}
		\end{align}
	\end{theorem}
	\begin{proof}
		\cite[Theorem~10]{WuVerdu12_fmmseMI} states that if $\|Y\|\le b$, then for any $(U,V)$ jointly distributed with $Y$,
		\begin{align}
			R_2(Y|U) - R_2(Y|U,V) 
			&\le 2 b^2 I(V;Y|U) \label{eq:R2_mi_pf0}  .
		\end{align}
		Using this result and the fact that $R_2(Y|X,W) = R_2(Y|X,W,Z^n)$,  we obtain \eqref{eq:R2_mi}.
	\end{proof}
	\noindent With Theorem~\ref{prop:regression_quadratic_MI}, all the upper bounds on ${\rm MER}_{\log}$ derived in Section~\ref{sec:MI_log} can be used to further upper-bound ${\rm MER}_{2}$.
	In particular, whenever \eqref{eq:mi_expan} holds, we have ${\rm MER}_{2} = O(b^2 d /2 n)$ as $n\rightarrow\infty$.

	\subsection{Zero-one loss}\label{sec:MER_01}
	Another loss function we consider for hard predictions is the zero-one loss $\ell(y,a) = {\bf 1}\{y \neq a\}$ with $\sY=\sA$.
	For any $V$ that statistically depends on $Y$, the conditional Bayes envelope with respect to the zero-one loss is $R_{01}(Y|V=v) = 1-\max_{y\in\sY}P_{Y|v}(y)$, the optimal action is the conditional mode $\argmax_{y\in\sY}P_{Y|v}(y)$, and the corresponding Bayes risk is
	\begin{align}\label{eq:MER_01_def}
		R_{01}(Y|V) = 1 - \E[\max\nolimits_{y\in\sY}P_{Y|V}(y)] ,
	\end{align}
	with expectation taken with respect to $V$.
	The MER for the zero-one loss is 
	\begin{align}\label{eq:MER_01}
		{\rm MER}_{01} =  \E[\max\nolimits_{y\in\sY}P_{Y|X,W}(y)] - 
		\E[\max\nolimits_{y\in\sY}P_{Y|X,Z^n}(y)] ,
	\end{align}
	where the expectations are taken with respect to the conditioning variables.
	In this case, as the loss function takes values in $[0,1]$, Theorem~\ref{thm:regression_genloss_MI} stated in the next subsection gives an upper bound for ${\rm MER}_{01}$ in terms of $I(W;Y|X,Z^n)$:
	\begin{corollary}
		\label{co:classification_01_MI} 
		For the zero-one loss,
		\begin{align}
			{\rm MER}_{01} 
			&\le \sqrt{ \frac{1}{2} I(W;Y|X,Z^n) } .
		\end{align}
	\end{corollary}
	\noindent From the results in Section~\ref{sec:MI_log}, we know that whenever \eqref{eq:mi_expan} holds, ${\rm MER}_{01} = O(\sqrt{d/n})$ as $n\rightarrow\infty$.
	
	\smallskip
	For the special case of binary classification, where $\sY=\{0,1\}$, the Bayes risk $R_{01}(Y|X,Z^n)$ is studied in \cite{HausslerVC94} and is upper-bounded in terms of $H(Y^n|X^n)$.
	When the model is realizable, that is, when $Y=g(X,W)$ with some generative function $g:\sX\times\sW\rightarrow\{0,1\}$, it is also observed in \cite{HausslerVC94} that $H(Y^n|X^n)$ can be further upper-bounded in terms of the VC dimension of the generative function class $\mathcal G = \{g(\cdot,w):\sX\rightarrow\{0,1\},w\in\sW\}$, defined as
	\begin{align}\label{eq:VC_def}
		V(\mathcal G) := \sup\Big\{n\in\mathbb{N}: \sup_{x^n\in\sX^n}\big|\big\{(g(x_1,w),\ldots,g(x_n,w)), w\in\sW \big\}\big| = 2^n\Big\} .
	\end{align}
The Sauer-Shelah lemma \cite{Sauer72,Shelah72} states that, if $V(\mathcal G) = d$, then for all $x^n\in\sX^n$,
\begin{align}\label{eq:Sauer-Shelah}
	\big|\big\{(g(x_1,w),\ldots,g(x_n,w)), w\in\sW \big\}\big| \le \sum_{k=1}^{d} {{n}\choose{k}} \le en^d .
\end{align}
As ${\rm MER}_{01} \le  R_{01}(Y|X,Z^n)$, the results in \cite{HausslerVC94} lead to the following MER upper bounds.
	\begin{theorem}\label{th:MER01_bc}
	If $\sY=\{0,1\}$, then
	\begin{align}\label{eq:MER01_bc_1}
		{\rm MER}_{01} \le \frac{1}{2}H(Y|X,Z^n) \le \frac{1}{2n}H(Y^n|X^n) .
	\end{align}
Moreover, if $Y=g(X,W)$ with some function $g:\sX\times\sW\rightarrow\sY$, and the function class $\mathcal G = \{g(\cdot,w):\sX\rightarrow\sY,w\in\sW\}$ has VC dimension $d$, then
\begin{align}\label{eq:MER01_bc_2}
	{\rm MER}_{01} \le O\Big( \frac{d}{2n} \Big)  \quad\text{as $n\rightarrow\infty$.}
\end{align}
These upper bounds also hold for $\frac{1}{2}{\rm MER}_{\log}$ in the same settings.
	\end{theorem}
\begin{proof}
	The proof of \eqref{eq:MER01_bc_1} is essentially drawn from \cite{HausslerVC94}.
Using our notation,
\begin{align}
{\rm MER}_{01} &\le  R_{01}(Y|X,Z^n) \label{eq:MER01_bin_rlz_1} \\
&=  \E\Big[\min_{y\in\{0,1\}} P[Y=y|X,Z^n]\Big] \label{eq:MER01_bin_rlz_12} \\
&\le  \E\Big[\frac{1}{2} h_2\big( P[Y=1|X,Z^n]\big) \Big] \label{eq:MER01_bin_rlz_2} \\
&=  \frac{1}{2} H( Y|X,Z^n) \\
&\le  \frac{1}{2n} H( Y^n|X^n) , \label{eq:MER01_bin_rlz_3}
\end{align}
where \eqref{eq:MER01_bin_rlz_1} follows from the fact that ${R}_{01}(Y|X,W) \ge 0$; \eqref{eq:MER01_bin_rlz_12} follows from \eqref{eq:MER_01_def} and the assumption that $\sY=\{0,1\}$; \eqref{eq:MER01_bin_rlz_2} follows from the fact that $\min\{p,1-p\} \le \frac{1}{2}h_2(p)$ for $p\in[0,1]$, where $h_2(\cdot)$ is the binary entropy function; and \eqref{eq:MER01_bin_rlz_3} can be proved by the chain rule of Shannon entropy and the fact that $H(Y_i|X^n,Y^{i-1})$ decreases as $i$ increases, similar to the proof of Theorem~\ref{th:MER_log_I(W;Y^n|X^n)}.

The proof of \eqref{eq:MER01_bc_2} relies on the observation made in \cite{HausslerVC94} that $H(Y^n|X^n) \le d\log n + 1$ under a realizable model whenever the VC dimension of $\mathcal G$ is $d$, which is due to the Sauer-Shelah lemma \eqref{eq:Sauer-Shelah}. Additionally, from $H(Y^n|X^n) = \sum_{i=1}^n H(Y_i|X^n, Y^{i-1})$ and Lemma~\ref{lm:sum_log_n}, we have
\begin{align}
\lim_{n\rightarrow\infty} (n+1) H(Y|X,Z^n) = \lim_{n\rightarrow\infty}\frac{H(Y^n|X^n)}{\log n} \le d 
\end{align}
whenever these limits exist, which proves \eqref{eq:MER01_bc_2}.
 
The upper bounds also hold for $\frac{1}{2}{\rm MER}_{\log}$ because ${\rm MER}_{\log} \le H(Y|X,Z^n)$, as $H(Y|X,W) \ge 0$ when $\sY$ is discrete.
\end{proof}
In Section~\ref{sec:MER_rlz}, we discuss the MER under realizable models in more general settings, where the results go beyond binary classification and zero-one loss.

	\subsection{General loss functions}\label{sec:MER_gen_loss}
	Now we derive a general upper bound for the MER with respect to a wide range of loss functions.
	For an arbitrary loss function $\ell: \sY\times \sA \rightarrow \R$, let $\Psi^*:\sX \times \sW \rightarrow \sY$ be the optimal omniscient decision rule such that $\E[\ell(Y,\Psi^*(X,W))]=R_\ell(Y|X,W)$. 
	Given $(X,Z^n)$, let $W'$ be a sample from the posterior distribution $P_{W|X,Z^n}$ conditionally independent of everything else given $(X,Z^n)$. 
	Then the MER can be upper-bounded by
	\begin{align}
		{\rm MER}_{\ell} 
		&\le \E[\ell(Y,\Psi^*(X,W'))] - \E[\ell(Y,\Psi^*(X,W))] . \label{eq:gen_ell_MER}
	\end{align}
	Here, $\Psi^*(X,W')$ is a plug-in decision rule, where we first estimate $W$ by $W'$ from $(X,Z^n)$, and then plug $W'$ in $\Psi^*$ to predict $Y$ given $X$.
	The right side of \eqref{eq:gen_ell_MER} is the excess risk of this plug-in decision rule.
	Under regularity conditions on the moment generating function of $\ell(Y,\Psi^*(X,W'))$ under the conditional distribution $P_{Y,W'|X,Z^n}$, we have the following upper bound on ${\rm MER}_{\ell}$ in terms of $I(\Psi^*(X,W);Y|X,Z^n)$.
	\begin{theorem}\label{thm:regression_genloss_MI}
		Assume there is a function $\varphi(\lambda)$ defined on $[0,b)$ for some $b\in(0,\infty]$, such that 
		\begin{align}\label{eq:cgf_constraint_mi}
			\log \E_{x,z^n}\Big[ \exp\Big\{-\lambda \Big(\ell(Y,\Psi^*(x,W')) - \E_{x,z^n}\big[\ell(Y, \Psi^*(x,W'))  \big] \Big)\Big\} \Big] 
			\le \varphi(\lambda)
		\end{align}
		for all $0\le \lambda < b$ and all $(x,z^n)$, where $\E_{x,z^n}[\cdot]$ denotes the conditional expectation with respect to $(Y,W')$ given $(X,Z^n)=(x,z^n)$. Then
		\begin{align}
			{\rm MER}_{\ell} 
			&\le \varphi^{*-1}\left(I(\Psi^*(X,W);Y|X,Z^n)\right) ,  \label{eq:genloss_MI_gen}
		\end{align}
	where $\varphi^*(\gamma) \deq \sup_{0\le \lambda < b} \{\lambda \gamma - \varphi(\lambda)\}$, $\gamma\in\R$, is the Legendre dual of $\varphi$, and $\varphi^{*-1}(u) \deq \sup \{ \gamma \in\R: \varphi^*(\gamma) \le u \}$,  $u\in\R$, is the generalized inverse of $\varphi^*$. 
		In addition, if $\varphi(\lambda)$ is strictly convex over $(0, b)$ and $\varphi(0) =  \varphi'(0)=0$, then 
		$
		\lim_{x\downarrow 0}\varphi^{*-1}(x) = 0 .
		$
	\end{theorem}
	\begin{proof}
		We have the following chain of inequalities:
		\begin{align}
			{\rm MER}_{\ell}
			&\le \E[\ell(Y,\Psi^*(X,W'))] - \E[\ell(Y,\Psi^*(X,W))] \label{eq:regression_genloss_MI_gen2} \\
			&= \E\big[ \E[\ell(Y,\Psi^*(X,W')) - \ell(Y,\Psi^*(X,W))|X,Z^n] \big] \label{eq:regression_genloss_gen25} \\
			&\le \E\big[ \varphi^{*-1}\big( D_{\rm KL}( P_{Y,\Psi^*(X,W)|X,Z^n} \| P_{Y,\Psi^*(X,W')|X,Z^n}) \big) \big] \label{eq:regression_genloss_MI_gen3} \\
			&= \varphi^{*-1}\big( \E\big[ D_{\rm KL}( P_{Y,\Psi^*(X,W)|X,Z^n} \| P_{Y,\Psi^*(X,W')|X,Z^n}) \big] \big) \label{eq:regression_genloss_MI_gen4} \\
			&= \varphi^{*-1}\big( I(\Psi^*(X,W);Y|X,Z^n) \big) \label{eq:regression_genloss_MI_gen7} 
		\end{align}
		where
		\eqref{eq:regression_genloss_MI_gen3} follows from the assumption \eqref{eq:cgf_constraint_mi} in the statement of the theorem and Lemma~\ref{lm:DPQ_variational_gen} stated in Appendix~\ref{appd:pf_prop_genloss_MI} applied to $P=P_{Y,\Psi^*(x,W)|x,z^n}$ and $Q=P_{Y,\Psi^*(x,W')|x,z^n}$, and the expectation is taken with respect to $(X,Z^n)$; 
		\eqref{eq:regression_genloss_MI_gen4} follows from the concavity of $\varphi^{*-1}$, which is due to the convexity of $\varphi^*$, and Jensen's inequality;
		\eqref{eq:regression_genloss_MI_gen7} follows from the fact that $W'$ is conditionally i.i.d.\ of $W$ and conditionally independent of $Y$ given $(X,Z^n)$.
		The last claim of the theorem comes from the fact that under the assumptions on $\varphi$, its Legendre dual $\varphi^*$ is increasing on $[0,\infty)$ and continuous at $0$ and so is the inverse $\varphi^{*-1}$.
	\end{proof}
	An example for the condition in \eqref{eq:cgf_constraint_mi} to hold is when the random variable $\ell(Y,\Psi^*(x,W'))$ is \textit{$\sigma^2$-subgaussian}\footnote{A random variable $U$ is $\sigma^2$-subgaussian if $\E[e^{\lambda (U-\E U)}] \le e^{\lambda^2\sigma^2/2}$ for all $\lambda\in\R$.} conditionally on $(X,Z^n)=(x,z^n)$. In this case, \eqref{eq:cgf_constraint_mi} holds with $b = \infty$ and $\varphi(\lambda) = {\sigma^2\lambda^2}/{2}$, and we have the following corollary.
	\begin{corollary}\label{co:MER_mi_subG}
		If $\ell(Y,\Psi^*(x,W'))$ is \textit{$\sigma^2$-subgaussian} conditionally on $(X,Z^n)=(x,z^n)$ for all $(x,z^n)$, then
		\begin{align}
			{\rm MER}_{\ell} 
			&\le \sqrt{ {2\sigma^2} I(\Psi^*(X,W);Y|X,Z^n)} . \label{eq:genloss_MI_subgaussian}
		\end{align}	
	\end{corollary}
	\noindent Using the fact that if $\ell(\cdot,\cdot)\in[a,b]$ then $\ell$ is $(b-a)^2/4$-subgaussian under any distribution of the arguments, Corollary~\ref{co:MER_mi_subG} can provide upper bound for the MER under any bounded loss functions.
More generally, Theorem~\ref{thm:regression_genloss_MI} can be applied in the situation where the loss function is unbounded and non-subgaussian. In Appendix~\ref{appd:MER2_lin_mi}, we present such a case where an MER upper bound for the quadratic loss in linear regression is derived based on Theorem~\ref{thm:regression_genloss_MI}.

From the data processing inequality of mutual information, 
\begin{align}
	I(\Psi^*(X,W);Y|X,Z^n) \le I(W;Y|X,Z^n).
\end{align}
Since $\varphi^{*-1}$ defined in Theorem~\ref{thm:regression_genloss_MI} is an increasing function on $[0,\infty)$, the upper bounds in \eqref{eq:genloss_MI_gen} and \eqref{eq:genloss_MI_subgaussian} can be weakened by replacing $I(\Psi^*(X,W);Y|X,Z^n)$ with $I(W;Y|X,Z^n)$ or any of its upper bounds derived in Section~\ref{sec:MI_log}. In particular, when \eqref{eq:mi_expan} holds in addition with the assumption in Theorem~\ref{thm:regression_genloss_MI}, we have ${\rm MER}_\ell = O(\varphi^{*-1}(d/n))$ as $n\rightarrow\infty$.

	Theorem~\ref{thm:regression_genloss_MI}	
	also provides a connection between the MER and the mutual information between the \emph{observed data} and the \emph{learned model parameters}.
	If $X$ is independent of $W$, then $P_{W|X,Z^n}=P_{W|Z^n}$, and $(W',Z^n)$ have the same joint distribution as $(W,Z^n)$. In this case, when the condition in Corollary~\ref{co:MER_mi_subG} is satisfied, upper-bounding $I(W;Y|X,Z^n)$ in \eqref{eq:genloss_MI_subgaussian} by $\frac{1}{n}I(W;Z^n)$ according to Corollary~\ref{co:classification_log_mi_total} leads to the following result.
	\begin{corollary}\label{co:MER_mi_Z_W}
		If $X$ is independent of $W$ in addition to the condition in Corollary~\ref{co:MER_mi_subG}, then
		\begin{align}
			{\rm MER}_{\ell} 
			&\le \sqrt{\frac{2\sigma^2}{n}I(Z^n ; W')} ,
		\end{align}
		where $I(Z^n;W')$ is the mutual information between the data and the {learned model parameters} sampled from the posterior distribution $P_{W|Z^n}$.
	\end{corollary}
	\noindent 
	Corollary~\ref{co:MER_mi_Z_W} is an analogue of the generalization-information relationship in the frequentist learning \cite{RusZou15,XuRaginsky17}, where it is shown that the generalization error in frequentist learning can be upper-bounded in terms of the mutual information between the \emph{observed data} and the \emph{learned hypothesis}.
	From \eqref{eq:genloss_MI_gen}, we also know that when the more general condition in Theorem~\ref{thm:regression_genloss_MI} is satisfied, we have ${\rm MER}_{\ell} \le \varphi^{*-1}(\frac{1}{n}I(Z^n ; W'))$, which is analogous to upper bounds on the generalization error in \cite{JiaoHanWeissman_bias_17}.
	

	\subsection{Realizable models and connection to VC dimension}\label{sec:MER_rlz}
	In Section~\ref{sec:MER_01} we have presented the MER for the zero-one loss under the realizable model of binary classification. Here, we present a few results on the MER for general loss functions under general realizable models. 
	These results provide tighter asymptotic MER bounds under realizable models than directly using the general results obtained in the previous subsection.
	Following the observations made in \cite{HausslerVC94}, these results also show how the key quantities in classical frequentist learning theory, notably the Vapnik–Chervonenkis (VC) dimension, can be naturally brought into the MER analysis in Bayesian learning through the information-theoretic framework proposed in this work.
	
	A realizable model is a model where the quantity of interest $Y$ is determined by the observation $X$ and the model parameters $W$ through a \emph{generative function} $g:\sX\times\sW\rightarrow\sY$.
	Under a realizable model, $g(X,W')$ can serve as a plug-in decision rule, where $W'$ is a sample from the posterior distribution $P_{W|X,Z^n}$, conditionally independent of everything else given $(X,Z^n)$. 
	It is observed in a follow-up work \cite{RD_MER21} (which has appeared after the initial version of this paper was posted) that the generalization error bounds developed in \cite{steinke20_cmi} for the realizable setting of frequentist learning can be adapted to MER bounds for realizable models in Bayesian learning.
	In particular, \cite[Lemma~3]{RD_MER21} shows that for a loss function $\ell\in[0,b]$, if $R_\ell(Y|X,W)=0$, then ${\rm MER}_\ell \le 3bI(W;Y|X,Z^n)$.
	Following this approach, the next result provides an upper bound for the MER under realizable models, with a better prefactor and a tighter conditional mutual information term.
	\begin{theorem}\label{thm:MER_mi_rlz}
		For a loss function $\ell\in[0,b]$, if there exists a function $g:\sX\times\sW\rightarrow\sY$ such that $\ell(Y, g(X,W)) = 0$ almost surely with respect to the joint distribution $P_{W,X,Y}$, then
		\begin{align}
			{\rm MER}_{\ell} 
			\le \frac{b}{\log 2}I(g(W,X) ; Y|X,Z^n) . \label{eq:MER_mi_rlz}
		\end{align}
	\end{theorem}
	\begin{proof}
		We have the following chain of inequalities:
		\begin{align}
			{\rm MER}_{\ell}
			&\le \E[\ell(Y,g(X,W'))]  \label{eq:MER_MI_rlz_2} \\
			&= \E\big[ \E[\ell(Y,g(X,W'))|X,Z^n] \big]  \label{eq:MER_MI_rlz_25} \\
			&\le \int \frac{b}{\log 2} I(g(x,W);Y | X=x, Z^n = z^n) P_{X,Z^n}({\rm d}x, {\rm d}z^n) \label{eq:MER_MI_rlz_3} \\
			&= \frac{b}{\log 2} I(g(X,W);Y|X,Z^n) \label{eq:MER_MI_rlz_4} ,
		\end{align}
		where
		\eqref{eq:MER_MI_rlz_2} follows from the assumption that $\ell(Y,g(X,W)) = 0$, the minimum loss, which implies that $R_\ell(Y|X,W)=0$; \eqref{eq:MER_MI_rlz_3} follows by applying Lemma~\ref{lm:decouple_rlz} stated below to the joint distribution $P_{g(x,W),Y|x,z^n}$ for all $(x,z^n)$ under $P_{X,Z^n}$.
	\end{proof}
	The following lemma used in the proof of Theorem~\ref{thm:MER_mi_rlz} is adapted from \cite[Theorem~5.7]{steinke20_cmi}, where it is developed for bounding the generalization error in the realizable setting of frequentist learning. 
	\begin{lemma}\label{lm:decouple_rlz}
	Let $V$ and $Y$ be jointly distributed random variables on $\sV\times\sY$. Let $V'$ be an independent copy of $V$, that is, $P_{V'}=P_V$ and $V'$ is independent of $(V,Y)$. For a function $\ell:\sY\times\sV\rightarrow[0,b]$, if $\ell(Y,V)=0$ almost surely with respect to $P_{V,Y}$, then
	\begin{align}
		\E[\ell(Y,V')] \le \frac{b}{\log 2} I(V;Y) .
	\end{align}
	\end{lemma}
\begin{proof}
	We follow the symmetrization idea used in the proof of \cite[Theorem~5.7]{steinke20_cmi}.
Let $\tilde{V} = (V_0,V_1)$ with $V_0$ and $V_1$ being i.i.d.\ samples from $P_V$. Let $S$ and $S'$ be i.i.d.\ uniform Bernoulli random variables independent of $\tilde{V}$, with $\overline S = 1-S$, and $\overline {S'} = 1 - S'$. With these random variables at hand, we can construct $V=\tilde{V}_S$, $V'=\tilde{V}_{\overline S}$, and $Y$ to be jointly distributed with $V$ and conditionally independent of everything else given $V$.
Then, following the technique used in the proof of \cite[Theorem~5.7]{steinke20_cmi}, for any $u>0$ and $t>0$,
\begin{align}
	& \E[\ell(Y,V')] \nonumber \\ 
	=& \E[\ell(Y,V')] - \E[u \ell(Y,V)] \label{eq:pf_decouple_rlz_1} \\
	=& \E[\ell(Y,\tilde{V}_{\overline S})] - \E[u \ell(Y,\tilde{V}_S)] \\ 
	=& \E\big[ \E[\ell(Y,\tilde{V}_{\overline S}) - u \ell(Y,\tilde{V}_S) | \tilde{V}] \big] \\
	\le&  \frac{1}{t} \Big( I(S;Y|\tilde V) + \E\Big[ \log \E\big[\exp\big\{t\big(\ell(Y,\tilde{V}_{\overline{S'}}) - u \ell(Y,\tilde{V}_{S'})\big) \big\} \big | \tilde{V}\big] \Big] \Big) \label{eq:pf_decouple_rlz_2} \\
	=&  \frac{1}{t} \Big( I(S;Y|\tilde V) + \E\Big[ \log \E\big[ \E\big[\exp\big\{t\big(\ell(Y,\tilde{V}_{\overline{S'}}) - u \ell(Y,\tilde{V}_{S'})\big) \big\} |Y, S, \tilde{V} \big] \big | \tilde{V} \big] \Big] \Big) \label{eq:pf_decouple_rlz_3} \\
	=&  \frac{1}{t} \Big( I(S;Y|\tilde V) + \E\Big[ \log \E\big[ \tfrac{1}{2}e^{t\ell(Y,\tilde{V}_{\overline S})} + \tfrac{1}{2} e^{- u t \ell(Y,\tilde{V}_{\overline S})} \big | \tilde{V} \big] \Big] \Big)  \label{eq:pf_decouple_rlz_5} 
\end{align}
where \eqref{eq:pf_decouple_rlz_1} is due to the assumption that $\ell(Y,V) = 0$ almost surely; \eqref{eq:pf_decouple_rlz_2} is due to the Donsker-Varadhan theorem, which implies that
\begin{align}
	D_{\rm KL}(P_{S,Y|\tilde{V}} \| P_{S',Y|\tilde{V}}) \ge \E\big[t \big(\ell(Y,\tilde{V}_{\overline S}) - u \ell(Y,\tilde{V}_{S})\big) \big | \tilde V \big] - \log \E\big[\exp\big\{t\big(\ell(Y,\tilde{V}_{\overline{S'}}) - u \ell(Y,\tilde{V}_{S'})\big) \big\} \big | \tilde{V}\big] \nonumber ;
\end{align}
and \eqref{eq:pf_decouple_rlz_5} follows from the fact that $S'$ is equally likely to be $S$ or $\overline S$ conditional on $S$, writing out the inner-most expectation in this way, and by setting $\ell(Y,\tilde{V}_S)$ to $0$.

Setting $t = \frac{\log 2}{b}$ and sending $u\rightarrow\infty$, we see that the inner expectation in \eqref{eq:pf_decouple_rlz_5} is upper-bounded by $1$, which leads to
\begin{align}
	\E[\ell(Y,V')] \le \frac{b}{\log 2} I(S;Y|\tilde{V}) .
\end{align}
The claim follows from the observation made in \cite{Haghifam20} that $
I(S;Y|\tilde{V}) \le I(\tilde{V}, S;Y) 
= I(V;Y) $,
where the equality holds because $Y$ is conditionally independent of $(\tilde{V},S)$ given $V = \tilde{V}_S$.
\end{proof}

Theorem~\ref{thm:MER_mi_rlz} can be weakened by the data processing inequality of mutual information, 
\begin{align}\label{eq:MER_mi_dpi_rlz}
	I(g(X,W);Y|X,Z^n) \le I(W;Y|X,Z^n) ,
\end{align}
which can be further bounded by $\frac{1}{n}I(W;Y^n|X^n)$ or $\frac{1}{n}I(W;Z^n)$ due to Theorem~\ref{th:MER_log_I(W;Y^n|X^n)} or Corollary~\ref{co:classification_log_mi_total}. When $\sY$ is discrete, $I(W;Y^n|X^n)$ can be further bounded by $H(Y^n|X^n)$.
It implies that under a realizable model with discrete $\sY$, not necessarily binary, the MER with respect to a bounded loss function can be upper-bounded nonasymptotically on the order of $H(Y^n|X^n)/n$.

	With a realizable model, a natural question to ask is how the MER depends on the richness of the \emph{generative function class} $\mathcal G = \{g(\cdot,w):\sX\rightarrow\sY, w\in\sW\}$. When $\sY=\{0,1\}$, one featuring combinatorial quantity that measures the richness of $\mathcal G$ is its VC dimension, defined in \eqref{eq:VC_def}.
	The connection between $H(Y^n|X^n)$ and $V(\mathcal G)$ is observed in \cite{HausslerVC94}, in the setting of binary classification with the zero-one loss. We make use of it here to obtain a corollary of Theorem~\ref{thm:MER_mi_rlz}, which extends the results in Theorem~\ref{th:MER01_bc} as it applies to more general loss functions.
	\begin{corollary}\label{co:MER_VC}
		Under a realizable model with $\sY=\{0,1\}$, if the function class $\mathcal G = \{g(\cdot,w):\sX\rightarrow\sY, w\in\sW\}$ has VC dimension $d$, then for any loss function $\ell\in[0,b]$,
		\begin{align}\label{eq:MER_0b_rlz_bc}
			{\rm MER}_{\ell} \le  O\Big( \frac{b}{\log 2}\cdot\frac{d}{n} \Big) \quad \text{as $n\rightarrow\infty$} . 
		\end{align}
	\end{corollary}
\begin{proof}
By \eqref{eq:MER_mi_dpi_rlz} and Theorem~\ref{th:MER_log_I(W;Y^n|X^n)}, and the assumptions that the model is realizable and $\sY$ is discrete, the upper bound in Theorem~\ref{thm:MER_mi_rlz} can be weakened to
\begin{align}
	{\rm MER}_{\ell} 
	&\le \frac{b}{\log 2}H(Y|X,Z^n) 
	\le \frac{b}{n \log 2}H(Y^n|X^n)  \label{eq:MER_rlz_VC_2} .
\end{align}
The Sauer-Shelah lemma \eqref{eq:Sauer-Shelah} implies that $H(Y^n|X^n) \le d\log n + 1$. In addition, from the chain rule of Shannon entropy and Lemma~\ref{lm:sum_log_n}, we have
\begin{align}
\lim_{n\rightarrow\infty} (n+1) H(Y|X,Z^n) = \lim_{n\rightarrow\infty}\frac{H(Y^n|X^n)}{\log n} \le d ,
\end{align}
similar to the proof of \eqref{eq:MER01_bc_2} in Theorem~\ref{th:MER01_bc}. This proves the claim in view of \eqref{eq:MER_rlz_VC_2}.
\end{proof}
The VC dimension plays a key role in the frequentist learning theory, in bounding the excess risk in terms of the richness of the \emph{hypothesis class}, which amounts to the set of decision rules. In Bayesian learning, while there is no restriction on the decision rules, Corollary~\ref{co:MER_VC} shows that the VC dimension of the generative function class plays a similar role in upper-bounding the MER.
	
\section{Upper bounds via functional and distributional continuities}\label{sec:MER_W_est_error}
In the previous section, the upper bounds are derived by relating the MER to $I(W;Y|X,Z^n)$. 
In this section, we explore alternative methods for bounding the MER, either in terms of the smoothness of the optimal omniscient decision rule in model parameters, or in terms of the smoothness of the minimum expected loss in the predictive model. These smoothness, or continuity properties enable us to bound the MER via the accuracy of estimated parameters from the data. The following lemma is instrumental for this approach.
	\begin{lemma}[\cite{Liu17}, \cite{Bhatt18}]\label{lm:posterior_error} Let $(\sU,\rho)$ be a metric space. If $U$ and $U'$ are two random elements of $\sU$ that are conditionally i.i.d.\ given another random element $V$ of some space $\sV$, i.e., $P_{U,U'|V=v} = P_{U|V=v} P_{U'|V=v}$ and $P_{U|V=v}=P_{U'|V=v}$ for all $v\in\sV$, then
		$
		\E[\rho(U',U)] \le 2 R_\rho(U|V) .
		$
		Moreover, if $\sU=\R^d$, then 
		$
		\E[\| U'-U \|^2] = 2 R_2(U|V) .
		$
	\end{lemma}

As an aside, Lemma~\ref{lm:posterior_error} provides us with a means for evaluating the performance of making randomized prediction by sampling from the posterior predictive distribution $P_{Y|X,Z^n}$, via upper-bounding the corresponding MER. 
Let $Y'$ be sampled from $P_{Y|X,Z^n}$, which can be realized by first sampling $W'$ from $P_{W|X,Z^n}$ then $Y'$ from $P_{Y|X,W'}$. Then, for any metric $\ell$ on $\sY$, we have
\begin{align}
	\E[\ell (Y,Y')] 	\le 2R_\ell(Y|X,W) + 2 {\rm MER}_{\ell} .
\end{align}

	\subsection{Via continuity of optimal omniscient decision rule}
	\subsubsection{General upper bound}\label{sec:MER_cont_dec_rule_gen}
	We start from \eqref{eq:gen_ell_MER} which states that ${\rm MER}_{\ell} \le \E[\ell(Y,\Psi^*(X,W'))] - \E[\ell(Y,\Psi^*(X,W))]$, where $\Psi^*$ is the optimal omniscient decision that achieves $R_\ell(Y|X,W)$ when $W$ is known, and $W'$ is a sample from the posterior distribution $P_{W|X,Z^n}$ conditionally independent of everything else given $(X,Z^n)$. 
	The MER can be upper-bounded in terms of the smoothness of the function $\ell(y,\Psi^*(x,w))$ in $w$ and the accuracy of approximating $W$ by $W'$.
	\begin{theorem}\label{th:cont_Psi*}
		If $\sW=\R^d$ and $W$ is independent of $X$, then
		\begin{align}
			{\rm MER}_{\ell} \le \E\Big[\sup_{y\in\sY} \sup_{w\in\sW} \| \nabla_w \ell(y,\Psi^*(X,w))\|\Big] \sqrt{2 R_2(W|Z^n)} ,
		\end{align}	
	where $\Psi^*$ is the optimal omniscient decision rule for the loss function $\ell$. 
	\end{theorem}
	\begin{proof}
		We have
		\begin{align}
			{\rm MER}_{\ell} 
			&\le \E[\ell(Y,\Psi^*(X,W')) - \ell(Y,\Psi^*(X,W))] \\
			&\le \E\Big[ \sup_{w\in\sW} \| \nabla_w \ell(Y,\Psi^*(X,w)) \| \cdot \| W'-W \|\Big]   \\
			&\le \E\Big[\sup_{y\in\sY} \sup_{w\in\sW} \| \nabla_w \ell(y,\Psi^*(X,w)) \|\Big] \E[ \| W'-W \|]   \\
			&\le \E\Big[\sup_{y\in\sY} \sup_{w\in\sW} \| \nabla_w \ell(y,\Psi^*(X,w))\|\Big] \sqrt{2R_2(W|Z^n)} ,
		\end{align}	
		where we used \eqref{eq:gen_ell_MER}, Lemma~\ref{lm:Lip_multi}, the assumption that $X$ and $W$ are independent, and the fact that $\E[ \| W'-W \|]  \le \sqrt{E[ \| W'-W \|^2]} = \sqrt{2R_2(W|Z^n)}$ due to Lemma~\ref{lm:posterior_error}.
	\end{proof}
	
	\paragraph{Example: constant $\Psi^*$} \mbox{} 
	
	\noindent An extreme case where Theorem~\ref{th:cont_Psi*} can be useful is when $\Psi^*:\sX\times\sW\rightarrow\sY$ does not dependent on $W$ under certain loss functions, in which case Theorem~\ref{th:cont_Psi*} guarantees that the MER is zero. For example, if $Y_i = g(X_i , W)V_i$ and $Y= g(X,W)V$, with some $g:\sX\times\sW\rightarrow\R$ and $(V^n,V)$ being i.i.d.\ zero-mean random variables independent of $(W,X^n,X)$, then for the quadratic loss, 
	\begin{align}
	 \Psi^*(X,W) = \E[g(X,W)V|X,W] = g(X,W)\E[V] \equiv 0 ,
	\end{align}
	hence ${\rm MER}_2 = 0$ by Theorem~\ref{th:cont_Psi*}.
	It implies that for this example
	\begin{align}
		R_2(Y|X,Z^n) = R_2(Y|X,W) = \E\big[g(X,W)^2 V^2\big] = \E\big[g(X,W)^2\big]\Var[V] ,
	\end{align}
	which shows that a small MER does not necessarily mean a small Bayes risk $R_\ell(Y|X,Z^n)$.

	
\paragraph{Example: logistic regression}  \mbox{}

	\noindent Another example where Theorem~\ref{th:cont_Psi*} can be applied to is bounding the MER for logistic regression with the log loss.
	Bayesian logistic regression is an instance under the predictive modeling framework, where $\sY = \{0,1\}$, $W\in\R^d$ is assumed to be independent of $X$, and the predictive model is specified by 
	$
		K_{Y|x,w}(1) = \sigma( w^\top \phi(x)) 
	$, 
	with $\sigma(a) \deq 1/(1+e^{-a})$, $a\in\R$, being the logistic sigmoid function, and $\phi(x) \in \R^d$ being the feature vector of the observation.
	For the log loss, the optimal omniscient decision rule $\Psi^*(x,w)$ is the Bernoulli distribution with bias $\sigma( w^\top \phi(x))$.
	Since $|\frac{\rm d}{{\rm d}a}\log\sigma(a)| = |1 - \sigma(a)| \le 1$, from Theorem~\ref{th:cont_Psi*} we have 
	\begin{align}
		{\rm MER}_{\log} \le \E[\|\phi(X)\|] \sqrt{2 R_2(W|Z^n)} . \label{eq:MER_log_logistic_cont_Psi}
	\end{align}
This result explicitly shows that the MER in logistic regression depends on how well we can estimate the model parameters from data, as it is dominated by $R_2(W|Z^n)$, the MMSE of estimating $W$ from $Z^n$.
For logistic regression, a closed-form expression for this MMSE may not exist. 
Nevertheless, any upper bound on it that is nonasymptotic in $d$ and $n$ will translate to a nonasymptotic upper bound on the MER. 
In Section~\ref{sec:MER_example} we continue the discussion of this example with different upper-bounding methods, where the dependence on $R_2(W|Z^n)$ can be improved when it is small.

	\subsubsection{Realizable models with additive noise}\label{sec:MER_cont_g(x,w)}
	The smoothness of $\Psi^*(x,w)$ in $w$ can lead to potentially tighter MER bounds under realizable models, possibly with additive noise.
	Consider the generative model of the form $Y_i = g(X_i,W) + V_i$ and $Y = g(X,W) + V$, where the generative function $g:\sX\times\sW\rightarrow\R$ is some parametric nonlinearity in general, which could be approximated by a feedforward neural network, the parameter vector $W\in\R^d$ is independent of $(X^n,X)$, and the additive noise $(V^n,V)$ are i.i.d.\ real-valued random variables independent of $(W,X^n,X)$.
	This model encompasses both linear and nonlinear Bayesian regression problems. We have the following MER bounds for the quadratic loss under this model.
	\begin{theorem}\label{th:rlz_n}
		Under the model considered above, for the quadratic loss,
			\begin{align}
				{\rm MER}_{2} 
				&\le  2  R_2(g(X,W)|X,Z^n) 
				\le  2 \E\Big[ \sup_{w\in\sW} \|\nabla_w g(X,w) \|^2\Big] R_2(W|Z^n) .
		\end{align}
	\end{theorem}
\begin{proof}
Under the model considered above, for the quadratic loss, 
\begin{align}
\Psi^*(X,W) = \E[g(X,W)+V|X,W] = g(X,W)+\E[V] .
\end{align}
We have
\begin{align}
	{\rm MER}_{2} 
	&\le \E\big[(Y - g(X,W') - \E[V])^2\big] - \E\big[(Y - g(X,W) - \E[V])^2 \big] \label{eq:pf_rlz_n_1} \\
	&= \E\big[(g(X,W') - g(X,W))^2\big] + \Var[V] - \Var[V] \label{eq:pf_rlz_n_2} \\
	&= 2R_2(g(X,W)|X,Z^n) \label{eq:pf_rlz_n_5} \\
	&\le \E\Big[ \sup_{w\in\sW} \| \nabla_w g(X,w) \|^2 \cdot \| W'-W \|^2 \Big]  \label{eq:pf_rlz_n_3} \\
	&= 2 \E\Big[\sup_{w\in\sW} \| \nabla_w g(X,w)\|^2 \Big] R_2(W|Z^n) , \label{eq:pf_rlz_n_4} 
\end{align}	
where \eqref{eq:pf_rlz_n_1} follows from \eqref{eq:gen_ell_MER}; \eqref{eq:pf_rlz_n_2} follows from the independence between $(X,W)$ and $V$;
\eqref{eq:pf_rlz_n_5} follows from the fact that $g(X,W')$ and $g(X,W)$ are conditionally independent given $(X,Z^n)$, and  Lemma~\ref{lm:posterior_error}; \eqref{eq:pf_rlz_n_3} follows from \eqref{eq:pf_rlz_n_2} and Lemma~\ref{lm:Lip_multi}; and \eqref{eq:pf_rlz_n_4} follows from the independence between $X$ and $W$, and Lemma~\ref{lm:posterior_error}.
\end{proof}

\paragraph{Example: linear regression}  \mbox{}

\noindent Theorem~\ref{th:rlz_n} can be applied to bounding the MER of the linear regression problem with the quadratic loss.
Bayesian linear regression is an instance of the noisy realizable model considered above, where $g(x,w)=w^\top \phi(x)$, $\phi(x) \in \R^d$ is the feature vector of the observation $x$, and $(V^n,V)$ are i.i.d.\ samples from $\mathcal N(0, \sigma^2)$.
With the Gaussian prior of model parameters $P_W = \mathcal N(0,\sigma_W^2 \mathbf I_d)$, the MMSE for estimating $W$ from $Z^n$ has a closed-form expression
\begin{align}
	R_2(W|Z^n) = \E[{\rm tr} (C_{W|Z^n})] ,
\end{align}
where
\begin{align}\label{eq:C_W|Z^n}
	C_{W|Z^n} = \Big(\frac{1}{\sigma_W^{2}}\mathbf{I}_d + \frac{1}{\sigma^{2}}\mathbf\Phi\mathbf\Phi^\top \Big)^{-1}
\end{align}
is the conditional covariance matrix of $W$ given $Z^n$, which only depends on $X^n$ through the $d\times n$ feature matrix $\mathbf\Phi = [\phi(X_1),\ldots,\phi(X_n)]$. Under this model, we also have $\nabla_w g(x,w) = \phi(x)$, hence Theorem~\ref{th:rlz_n} implies that
	\begin{align}
		{\rm MER}_{2}
		\le 2 \E[\|\phi(X)\|^2] \E[{\rm tr} (C_{W|Z^n})]  . \label{eq:rlz_n_lin_MER2}
	\end{align}

Under the above model with Gaussian prior, it can be shown that the posterior predictive distribution $P_{Y|x,z^n}$ is Gaussian with variance $\sigma^2 + \phi(x)^\top C_{W|z^n}\phi(x)$. From this we can obtain exact expressions for the MER and alternative upper bounds:
\begin{align}
	&{\rm MER}_{\log} 
	= \frac{1}{2}\E\Big[ \log \Big(1 + \frac{1}{\sigma^2} \phi(X)^\top C_{W|Z^n} \phi(X) \Big) \Big]
	\le \frac{\E[\|\phi(X)\|^2]}{2\sigma^2}  \E[{\rm tr} (C_{W|Z^n})]  \label{eq:MER_log_lin_reg_1} ,
\end{align}
and
\begin{align}
	&{\rm MER}_{2} = \E\big[\phi(X)^\top C_{W|Z^n} \phi(X)\big]  
	\le  \E[\|\phi(X)\|^2] \E[{\rm tr} (C_{W|Z^n})] . \label{eq:MER_2_lin_reg_1} 
\end{align}
The upper bounds in \eqref{eq:MER_log_lin_reg_1} and \eqref{eq:MER_2_lin_reg_1} are justified by noting that 
\begin{align}
	\phi(X)^\top C_{W|Z^n}\phi(X) = \|C_{W|Z^n}^{1/2} \phi(X)\|^2 \le \sigma_1^2\big(C_{W|Z^n}^{1/2}\big) \|\phi(X)\|^2 \le {\rm tr}(C_{W|Z^n}) \|\phi(X)\|^2 ,
\end{align}
where $\sigma_1(\cdot)$ is the largest singular value of the underlying matrix.
A special choice of the $d$ feature functions composing the feature vector is such that they are orthonormal with respect to $P_X$, namely $\int_\sX \phi_i(x)\phi_j(x)P_X({\rm d}x)=\I\{i=j\}$ for $i,j\in\{1,\ldots,d\}$. In this case, $\Phi\Phi^\top \approx n \E[\phi(X)\phi(X)^\top]=n\mathbf I_d$, hence $\E[{\rm tr} (C_{W|Z^n})] \sim O(d/n)$, implying that ${\rm MER}_{2}$ scale with $d$ and $n$ as $O(d/n)$ according to \eqref{eq:MER_2_lin_reg_1}.
It further implies that the upper bound \eqref{eq:rlz_n_lin_MER2} given by Theorem~\ref{th:rlz_n} is order-optimal for vanishing MER. 
We continue the discussion of this example in Section~\ref{sec:MER_example}.

	\subsection{Deviation of posterior predictive distribution from true predictive model}
	Under the predictive modeling framework, the generative model is specified as $P_{Z|W} = P_{X|W} K_{Y|X,W}$, with a parametrized probability transition kernel $K_{Y|X,W}$ describing the true predictive model of $Y$ given $X$. An alternative method for upper-bounding the MER under this framework is by examining the deviation of the posterior predictive distribution $P_{Y|X,Z^n}$ from the true predictive model $K_{Y|X,W}$ in terms of a suitable \textit{convex statistical distance} between them.
	Here, by a convex statistical distance we mean any statistical distance $(P,Q) \mapsto D(P,Q)$ that is convex in the first argument when the second one is held fixed, or convex in the second argument while the first one is held fixed. 
	For example, any $f$-divergence, including the commonly used total variation distance, KL divergence and $\chi^2$-divergence, is jointly convex in both arguments \cite{LieVaj06}.
	As another example, consider the $p$th power of $p$-Wasserstein distance between two Borel probability measures $P$ and $Q$ on $\R^m$ with finite second moments \cite{Villani_topics}:
	\begin{align}
		\mathcal W_p^p(P,Q) \deq \inf_{\pi \in \Pi(P,Q)} \E_{(X,Y)\sim\pi} [\|X-Y\|^p],
	\end{align}
	where $\Pi(P,Q)$ denotes the collection of all {couplings} of $P$ and $Q$, i.e., Borel probability measures on $\R^m \times \R^m$ with marginals $P$ and $Q$. As shown in Lemma~\ref{lm:cvx_W_p^p} in Appendix~\ref{appd:cvx_W_p^p}, $(P,Q) \mapsto \mathcal W_p^p(P,Q)$ is also jointly convex.
	The following lemma is key for relating the deviation of $P_{Y|X,Z^n}$ from $K_{Y|X,W}$ to the estimation error of model parameters.
	\begin{lemma}\label{lm:ub_f_div} Let $W'$ be a sample from the posterior distribution $P_{W|X,Z^n}$, such that $W$ and $W'$ are conditionally i.i.d.\ given $(X,Z^n)$.
		Then for any $(w,x,z^n)$ and any statistical distance $D$ that is convex in the first argument,
		\begin{align}\label{eq:post_ub_f_div}
			D(P_{Y|x,z^n} , K_{Y|x,w}) 
			\le \E[D(K_{Y|x,W'} , K_{Y|x,w})|x,z^n]  ,
		\end{align}
		and consequently,
		\begin{align}\label{eq:post_ub_f_div_E}
			\E[ D(P_{Y|X,Z^n} , K_{Y|X,W}) ] \le \E[ D(K_{Y|X,W'} , K_{Y|X,W}) ] 
		\end{align}
		where the expectations are taken with respect to the joint distribution of $(W,X,Z^n,W')$.
		Similarly, for any $(w,x,z^n)$ and any statistical distance $D$ that is convex in the second argument,
		\begin{align}\label{eq:post_ub_f_div_cvx2}
			D(K_{Y|x,w} , P_{Y|x,z^n}) 
			\le \E[D(K_{Y|x,w} , K_{Y|x,W'} ) | x,z^n] ,
		\end{align}
		and consequently,
		\begin{align}\label{eq:post_ub_f_div_E_cvx2}
			\E[ D(K_{Y|X,W}, P_{Y|X,Z^n} ) ] \le \E[ D( K_{Y|X,W} , K_{Y|X,W'} ) ] .
		\end{align}
	\end{lemma}
	\begin{proof}
		From the joint distribution in \eqref{eq:joint_dist}, it follows that for any $(w, z^n,z)$,
		\begin{align}
			P_{Y|x,z^n}(y) &= \int_{\sW} K_{Y|x,w'}(y) P_{W|x,z^n}({\rm d}w')
		\end{align}
		If the statistical distance considered here is convex in the first argument, we have
		\begin{align}
			D(P_{Y|x,z^n} , K_{Y|x,w}) \le \int_{\sW} D(K_{Y|x,w'} , K_{Y|x,w} ) P_{W|x,z^n}({\rm d}w') ,
		\end{align}
		which proves \eqref{eq:post_ub_f_div}.
		Taking expectations over the conditioning terms, we obtain \eqref{eq:post_ub_f_div_E}.
		The proof of \eqref{eq:post_ub_f_div_cvx2} and \eqref{eq:post_ub_f_div_E_cvx2} follows the same argument when $D$ is convex in the second argument.
	\end{proof}
	
	Whenever the convex statistical distance $D(K_{Y|x,w'} , K_{Y|x,w})$ can be upper-bounded via $\|w'-w\|$ or $\|w'-w\|^2$, Lemma~\ref{lm:posterior_error} can be used to further upper-bound $\E[ D(K_{Y|X,W'} , K_{Y|X,W}) ] $ in terms of the minimum achievable estimation error of $W$.
	In the following two subsections, we use two different methods together with Lemma~\ref{lm:ub_f_div} and Lemma~\ref{lm:posterior_error} to convert upper bounds on the deviation of $P_{Y|X,Z^n}$ from $K_{Y|X,W}$ into upper bounds on the MER for various loss functions.
	
	\subsection{From deviation of posterior predictive distribution to MER bound}\label{sec:MER_posterior_deviation}
	
	\subsubsection{Via conditional mutual information upper bound}\label{sec:MER_ub_cond_I}
	For the log loss, we can directly upper-bound $I(W;Y|X,Z^n)$ in terms of the KL divergence between $K_{Y|X,W}$ and $P_{Y|X,Z^n}$, and arrive at the following result with Lemma~\ref{lm:ub_f_div}.
	\begin{theorem}\label{th:MER_log_KL}
		When $P_{Z|W}=P_{X|W} K_{Y|X,W}$, let $W'$ be a sample from the posterior distribution $P_{W|X,Z^n}$, conditionally independent of everything else given $(X,Z^n)$. Then,
		\begin{align}
			{\rm MER}_{\log} 
			&\le \E[D_{\rm KL}( K_{Y|X,W} \| K_{Y|X,W'})] 
		\end{align}
		where the expectation is taken with respect to the joint distribution of $(W,W',X)$.
	\end{theorem}
	\begin{proof}
		From \eqref{eq:log_cond_mi}, we have
		\begin{align}
			{\rm MER}_{\log}  
			&= I(W ; Y| X, Z^n) \\
			&= \E[D_{\rm KL}(P_{Y|X, Z^n, W} \| P_{Y|X,Z^n})] \\
			&= \E[D_{\rm KL}(P_{Y|X, W} \| P_{Y|X,Z^n})]  \label{eq:classification_log_mi2} \\
			&\le \E[D_{\rm KL}(K_{Y|X,W} \| K_{Y|X, W'})] ,  \label{eq:classification_log_mi3} 
		\end{align}
		where \eqref{eq:classification_log_mi2} follows from the fact that $Y$ is conditionally independent of $Z^n$ given $(X,W)$; and \eqref{eq:classification_log_mi3} is from Lemma~\ref{lm:ub_f_div} and the fact that $D_{\rm KL}(P\| Q)$ is convex in $Q$ for a fixed $P$. 
	\end{proof}
	\noindent 
	
	In Section~\ref{sec:MER_example} we continue with the example of logistic regression, where Theorem~\ref{th:MER_log_KL} can be used with Lemma~\ref{lm:posterior_error} to bound the MER in terms of the MMSE of estimating $W$ from data.

	\subsubsection{Via continuity of generalized entropy}\label{sec:ub_cont_entropy}
	The second method for relating the MER to the deviation of posterior predictive distribution is directly comparing $R_{\ell}(Y|X,Z^n)$ against $R_\ell(Y|X,W)$, via the distributional continuity of the generalized entropy.
	We examine classification and regression problems separately.
	\paragraph{Classification}\mbox{} 
	
	\noindent
	For classification problems where $\sY$ is finite, we consider both the soft classification with the log loss and the hard classification with the zero-one loss. The MER upper bounds rely on the continuity properties of the Shannon entropy and the maximal probability, respectively, as stated in the following lemma, with proofs provided in Appendix~\ref{appd:pf_H_cont_miny_TV}. 
	For more general discussions on the continuity of generalized entropy, the reader may refer to \cite{Xu_cont_H_ISIT20,Xu_cont_H_learning}. 
	\begin{lemma}\label{lm:H_cont_miny_TV}
		Let $P$ and $Q$ be distributions on a finite set $\sY$ such that $\min_{y\in\sY}Q(y)>0$. Then
		\begin{align}
			& H_{}(P) - H_{}(Q) \le \big(-\log{\min_{y\in\sY}Q(y)}\big) d_{\rm TV}(P,Q) ,  \label{eq:H_cont_TV} \\
			& \max_{y\in\sY}Q(y) - \max_{y\in\sY}P(y) \le  d_{\rm TV}(P,Q) ,
		\end{align}
		where $d_{\rm TV}(P,Q)\deq \frac{1}{2}\sum_{y\in\sY} |P(y)-Q(y)|$ is the total variation distance between $P$ and $Q$.
	\end{lemma}
	\noindent Compared with the well-known Shannon entropy difference bound in terms of the total variation distance $|H(P)-H(Q)| \le 2 d_{\rm TV}(P , Q) \log ({|\sY|}/2 d_{\rm TV}(P , Q))$ when $d_{\rm TV}(P,Q)\le{1}/{4}$ \cite[Theorem~17.3.3]{Cover_book}, the bound given in \eqref{eq:H_cont_TV} is not as tight in $|\sY|$, but is tighter in $d_{\rm TV}(P,Q)$, which leads to sharper MER bounds when the data size is large.
	Armed with Lemma~\ref{lm:H_cont_miny_TV} and Lemma~\ref{lm:ub_f_div}, we have the following MER bounds for classification problems.
	\begin{theorem}\label{thm:classification_log_TV}
		If $\sY$ is finite, then for the log loss,
		\begin{align}
			{\rm MER}_{\log} 
			\le  \sup_{ x\in\sX ,\, w\in\sW}(- \log \kappa(x,w))  \E[d_{\rm TV}( K_{Y|X,W'}, K_{Y|X,W})] ,
		\end{align}
		where $\kappa(x,w) \deq \min_{y\in\sY}K_{Y|x,w}(y)$, $W'$ is a sample from $P_{W|X,Z^n}$, conditionally independent of everything else given $(X,Z^n)$, and the expectation is with respect to $P_{W,W',X}$. 
		In addition, for the zero-one loss,
		\begin{align}
			{\rm MER}_{01} 
			&\le \E[d_{\rm TV}( K_{Y|X,W'}, K_{Y|X,W})] .
		\end{align}
	\end{theorem}
	\begin{proof}
		When $\sY$ is finite, for the log loss,
		\begin{align}
			{\rm MER}_{\log} &= H(Y|X,Z^n) - H(Y|X,W) \\
			&= \int \big(H(Y|x,z^n) - H(Y|x,w) \big) P({\rm d}w, {\rm d}x, {\rm d}z^n) \\
			&\le \int \Big(-\log{\min_{y\in\sY}K_{Y|x,w}(y)}\Big) d_{\rm TV}(P_{Y|x,z^n},P_{Y|x,w}) P({\rm d}w, {\rm d}x, {\rm d}z^n) \label{eq:pf_log_TV_1}\\
			&\le \sup_{w\in\sW,x\in\sX}\big(-\log{\min\nolimits_{y\in\sY}K_{Y|x,w}(y)}\big) \E\big[d_{\rm TV}(P_{Y|X,Z^n},P_{Y|X,W})\big]  \label{eq:pf_log_TV_2} \\
			&\le \sup_{w\in\sW,x\in\sX}\big(-\log{\min\nolimits_{y\in\sY}K_{Y|x,w}(y)}\big) \E\big[d_{\rm TV}(P_{Y|X,W'},P_{Y|X,W})\big]  \label{eq:pf_log_TV_3} 
		\end{align}
		where \eqref{eq:pf_log_TV_1} follows from Lemma~\ref{lm:H_cont_miny_TV}; and \eqref{eq:pf_log_TV_3} follows from Lemma~\ref{lm:ub_f_div}.
		
		For the zero-one loss,
		\begin{align}
			{\rm MER}_{01} &= \E[\max\nolimits_{y\in\sY}K_{Y|X,W}(y)] - \E[\max\nolimits_{y\in\sY}P_{Y|X,Z^n}(y)]  \\
			&= \int \big( \max\nolimits_{y\in\sY}K_{Y|x,w}(y) - \max\nolimits_{y\in\sY}P_{Y|x,z^n}(y) \big) P({\rm d}w, {\rm d}x, {\rm d}z^n) \\
			&\le  \int d_{\rm TV}(K_{Y|x,w}, P_{Y|x,z^n} ) P({\rm d}w, {\rm d}x, {\rm d}z^n) \label{eq:pf_ub_R01_1} \\
			&\le  \E[d_{\rm TV}( K_{Y|X,W}, K_{Y|X,W'})] \label{eq:pf_ub_R01_2}
		\end{align}
		where  \eqref{eq:pf_ub_R01_1} follows from Lemma~\ref{lm:H_cont_miny_TV}, and \eqref{eq:pf_ub_R01_2} follows from Lemma~\ref{lm:ub_f_div}.
	\end{proof}

	\paragraph{Regression}\mbox{}
	
	\noindent
	Next, we consider regression problems with $\sY \subset \R^p$ under the assumption that both the marginal and various conditional distributions of $Y$ are absolutely continuous with respect to the Lebesgue measure. We consider both the soft prediction with the log loss, and the hard prediction with the quadratic loss. 
	For the soft setting, ${\rm MER}_{\log}$ can be upper-bounded using the continuity of differential entropy with respect to the Wasserstein distance, as stated in the following lemma.
	\begin{lemma}[\cite{PoliWu06}]\label{lm:h_cont}
		Let $U$ be a random vector in $\R^p$ with finite $\E[\|U\|^2]$, and $V$ be a Gaussian random vector in $\R^p$ with covariance matrix $\sigma^2 \mathbf I_p$.
		Then 
		\begin{align}
			h(U) - h(V) 
			&\le \frac{1}{2\sigma^2}\big(3\sqrt{\E[\|U\|^2]} + 11\sqrt{\E[\|V\|^2]} \big) \mathcal W_2(P_U, P_V)
		\end{align}
		where $\mathcal W_2(P_U,P_V)$ is the $2$-Wasserstein distance between $P_U$ and $P_V$. 
	\end{lemma}
	\noindent With Lemma~\ref{lm:h_cont} and Lemma~\ref{lm:ub_f_div}, we have the following bound for regression with the log loss.
	\begin{theorem}\label{thm:regression_log}
		If $\sY = \R^p$, 
		and $K_{Y|x,w}$ is Gaussian with covariance matrix $\sigma^2 \mathbf I_p$ for all $(x,w)$, then for the log loss,
		\begin{align}
			{\rm MER}_{\log} 
			&\le \frac{7}{\sigma^2} \sqrt{ \E\big[\|Y\|^2\big] \E[\mathcal W_2^2(K_{Y|X,W'} , K_{Y|X,W})] } ,
		\end{align}
		where $W$ and $W'$ are conditionally i.i.d.\ given $(X,Z^n)$, and the expectation is with respect to $P_{X,W,W'}$.
	\end{theorem}
	\begin{proof}
		For the log loss,
		\begin{align}
			{\rm MER}_{\log} &= h(Y|X,Z^n) - h(Y|X,W) \\
			&= \int \big(h(Y|x,z^n) - h(Y|x,w) \big) P({\rm d}w, {\rm d}x, {\rm d}z^n) \\
			&\le \int \Big(\frac{3}{2\sigma^2}\sqrt{\E\big[\|Y\|^2|x,z^n\big]} + \frac{11}{2\sigma^2}\sqrt{\E\big[\|Y\|^2|x,w\big]} \Big) \nonumber \\
			&\qquad\quad \mathcal W_2(P_{Y|x,z^n}, K_{Y|x,w})  P({\rm d}w, {\rm d}x, {\rm d}z^n) \label{eq:pf_ub_Rh_1} \\
			&\le \left(\int \Big(\frac{3}{2\sigma^2}\sqrt{\E\big[\|Y\|^2|x,z^n\big]} + \frac{11}{2\sigma^2}\sqrt{\E\big[\|Y\|^2|x,w\big]} \Big)^2  P({\rm d}w, {\rm d}x, {\rm d}z^n) \right)^{1/2} \nonumber \\
			&\quad\, \left( \int \mathcal W_2^2(P_{Y|x,z^n}, K_{Y|x,w})  P({\rm d}w, {\rm d}x, {\rm d}z^n) \right)^{1/2} \label{eq:pf_ub_Rh_2} \\
			&\le  \frac{7}{\sigma^2} \big( \E\big[\|Y\|^2\big] \big)^{1/2} 
			\left( \int \mathcal W_2^2(P_{Y|x,z^n} , K_{Y|x,w})  P({\rm d}w, {\rm d}x, {\rm d}z^n) \right)^{1/2} \label{eq:pf_ub_Rh_3} \\
			&\le \frac{7}{\sigma^2} \sqrt{ \E\big[\|Y\|^2\big] \E[\mathcal W_2^2(K_{Y|X,W'} , K_{Y|X,W})]\big) } . \label{eq:pf_ub_Rh_6}
		\end{align}
		where \eqref{eq:pf_ub_Rh_1} follows from Lemma~\ref{lm:h_cont}; \eqref{eq:pf_ub_Rh_2} follows from Cauchy-Schwarz inequality; \eqref{eq:pf_ub_Rh_3} follows from the triangle inequality of the $L_2$ norm, which states that $\sqrt{\E[(U+V)^2]} \le \sqrt{\E[U^2]} + \sqrt{\E[V^2]}$;
		and \eqref{eq:pf_ub_Rh_6} follows from Lemma~\ref{lm:ub_f_div} and Lemma~\ref{lm:cvx_W_p^p}.
	\end{proof}
	
	For the hard setting, in scalar regression problems with $\sY = \sA = \R$ and the quadratic loss, the MER as given by \eqref{eq:R2_Var} is the expected difference between two variances.
	The following results relate the variance difference between two probability distributions to their $2$-Wasserstein distance and KL divergence respectively.
	\begin{lemma}[\cite{Wu_Var}, proof given in Appendix~\ref{appd:pf_lm_Var_cont}]\label{lm:Var_cont}
		Let $U$ and $V$ be random variables over a set $\sU\subset\R$ with finite $\E[U^2]$ and $\E[V^2]$. Then,
		\begin{align}\label{eq:ub_Var_W2}
			|\Var[U]-\Var[V]| \le 2 \big(\sqrt{\E[U^2]} + \sqrt{\E[V^2]}\big) \mathcal W_2(P_U , P_V) .
		\end{align}
	\end{lemma}
	\noindent When $V$ is Gaussian with variance $\sigma^2$, Lemma~\ref{lm:Var_cont} with Talagrand’s inequality \cite{Talagrand96} states that
	\begin{align}
		|\Var[U]-\Var[V]| \le 2 \big(\sqrt{\E[U^2]} + \sqrt{\E[V^2]}\big) \sqrt{2\sigma^2 D_{\rm KL}( P_U \| P_V)} ;
	\end{align}
	under the same condition, we also have a potentially tighter bound \cite{Xu_cont_H_learning}:
	\begin{align}
		|\Var[U]-\Var[V]| \le 2\sigma^2 \big(\sqrt{D_{\rm KL}(P_U \| P_V)} + D_{\rm KL}(P_U \| P_V) \big) .
	\end{align}
	With Lemma~\ref{lm:Var_cont}, we can derive the following upper bounds for ${\rm MER}_{2}$.
	\begin{theorem}\label{prop:regression_quadratic}
		For regression problems with $\sY = \R$, if $\E[Y^2|x,w]$ is finite for all $(x,w)$, then for the quadratic loss,
		\begin{align}
			{\rm MER}_{2} &\le  4 \sqrt{ \E[Y^2] \E[\mathcal W_2^2(K_{Y|X,W'} , K_{Y|X,W})] } ,
		\end{align}
		where $W'$ is a sample from $P_{W|X,Z^n}$, conditionally independent of everything given $(X,Z^n)$.

	\end{theorem}
	\begin{proof}
		Similar to the proof of Theorem~\ref{thm:regression_log}, for the quadratic loss, we have
		\begin{align}
			{\rm MER}_{2} &= \int \big( \Var[Y|x,z^n] - \Var[Y|x,w]\big) P({\rm d}w, {\rm d}x, {\rm d}z^n) \\ 
			&\le 2 \int \big(\sqrt{\E[Y^2|x,z^n]} + \sqrt{\E[Y^2|x,w]} \big) \mathcal W_2(P_{Y|x,z^n}, K_{Y|x,w})  P({\rm d}w, {\rm d}x, {\rm d}z^n) \label{eq:pf_ub_R2_1} \\
			&\le 2\left(\int \big(\sqrt{\E[Y^2|x,z^n]} + \sqrt{\E[Y^2|x,w]} \big)^2 P({\rm d}w, {\rm d}x, {\rm d}z^n) \right)^{1/2} \cdot \nonumber \\
			&\qquad\! \left( \int \mathcal W_2^2(P_{Y|x,z^n}, K_{Y|x,w})  P({\rm d}w, {\rm d}x, {\rm d}z^n) \right)^{1/2} \label{eq:pf_ub_R2_2} \\
			&\le  4 \sqrt{ \E[Y^2] \E[\mathcal W_2^2(P_{Y|X,Z^n} , K_{Y|X,W})] } \label{eq:pf_ub_R2_5} \\
			&\le  4 \sqrt{ \E[Y^2] \E[\mathcal W_2^2(P_{Y|X,W'} , K_{Y|X,W})] } . \label{eq:pf_ub_R2_6}
		\end{align}
		where \eqref{eq:pf_ub_R2_1} follows from Lemma~\ref{lm:Var_cont}; \eqref{eq:pf_ub_R2_2} follows from Cauchy-Schwarz inequality; \eqref{eq:pf_ub_R2_5} follows from the triangle inequality of the $L_2$ norm; and \eqref{eq:pf_ub_R2_6} follows from Lemma~\ref{lm:ub_f_div} and Lemma~\ref{lm:cvx_W_p^p}.
	\end{proof}
	In Section~\ref{sec:MER_example} we make use of Theorem~\ref{thm:classification_log_TV}, \ref{thm:regression_log} and \ref{prop:regression_quadratic} with Lemma~\ref{lm:posterior_error} to bound the MER in concrete learning problems in terms of the MMSE of estimating $W$ from data.

\subsubsection{Examples} \label{sec:MER_example} 

\paragraph{Logistic regression (continued)}\mbox{}

\noindent
We continue with the example of logistic regression discussed in Section~\ref{sec:MER_cont_dec_rule_gen}, where $\sY = \{0,1\}$, $\sW\subset\R^d$, $
	K_{Y|x,w}(1) = \sigma( w^\top \phi(x))$
with $\sigma(a) \deq 1/(1+e^{-a})$, and $X$ is assumed to be independent of $W$.
As $\| \nabla_w \sigma(w^\top \phi(x)) \| \le \|\phi(x)\|/4$, from Lemma~\ref{lm:Lip_multi} we know that $ \sigma(w^\top \phi(x)) $ is $ \|\phi(x)\|/4 $-Lipschitz in $w$, hence
\begin{align}
	d_{\rm TV}(K_{Y|x,w'}, K_{Y|x,w}) 
	= \big| \sigma(w'^\top \phi(x)) -  \sigma(w^\top \phi(x)) \big| 
	\le \frac{1}{4} \|\phi(x)\| \|w' - w\| \label{eq:TV_logistic} . 
\end{align}
Then, from Theorem~\ref{th:MER_log_KL}, Theorem~\ref{thm:classification_log_TV} and Lemma~\ref{lm:posterior_error}, the following bounds hold for the log loss and the zero-one loss.
\begin{corollary}\label{prop:ub_excess_logistic}
	In binary logistic regression, for the log loss,
	\begin{align}\label{eq:logistic_log_ub}
		{\rm MER}_{\log}
		&\le  \log\Big(1+ \frac{1}{2} s_\phi^2 e^{s_\phi s_\sW} R_2(W|Z^n) \Big) 		
	\end{align}
	where $s_\phi \deq \sup_{ x\in\sX}\|\phi(x)\|$ and $s_\sW \deq \sup_{w\in\sW}\|w\|$;
	while for the zero-one loss, 
	\begin{align}
		{\rm MER}_{01}
		&\le   \frac{1}{4} \E[\|\phi(X)\|] \sqrt{2R_2(W|Z^n)} .
	\end{align}
\end{corollary}
\begin{proof}
	For the log loss, from Theorem~\ref{th:MER_log_KL},
	\begin{align}
		{\rm MER}_{\log} 
		&\le \E\big[D_{\rm KL}( K_{Y|X,W} \| K_{Y|X,W'})\big] \\
		&\le \E\Big[\log\Big(1+\frac{2d_{\rm TV}^2( K_{Y|X,W}, K_{Y|X,W'})}{\min\{\sigma(W'^\top \phi(X)), 1-\sigma(W'^\top \phi(X))\}} \Big)\Big] \\
		&\le \E\Big[\log\Big(1+ 4e^{|W'^\top \phi(X)|}{d_{\rm TV}^2( K_{Y|X,W}, K_{Y|X,W'})} \Big)\Big] \\
		&\le \E\Big[\log\Big(1+ \frac{1}{4}e^{s_\phi s_\sW} s_\phi^2 \|W-W'\|^2 \Big)\Big] \\
		&\le \log\Big(1+ \frac{1}{2}e^{s_\phi s_\sW} s_\phi^2 R_2(W|Z^n)) \Big) \\
		&\le \frac{1}{2} s_\phi^2 e^{s_\phi s_\sW} R_2(W|Z^n) \label{eq:logistic_log_ub_1}
	\end{align}
	where we used a reverse Pinsker's inequality \cite[Theorem~28]{fdiv_Sas_Ver}, the fact that $\max\{1/ \sigma(w^\top\phi(x)),1/ (1-\sigma(w^\top \phi(x))) \} \le 2\exp\{|w^\top \phi(x)|\}$,  $d_{\rm TV}(K_{Y|x,w'}, K_{Y|x,w}) \le \|\phi(x)\| \|w'-w\|/4$ from \eqref{eq:TV_logistic}, Jensen's inequality, and Lemma~\ref{lm:posterior_error}.

	For the zero-one loss, from Theorem~\ref{thm:classification_log_TV}, 
	\begin{align}
		{\rm MER}_{01} 
		&\le \E[d_{\rm TV}( K_{Y|X,W'}, K_{Y|X,W})] \\
		&\le \frac{1}{4} \E\big[\|\phi(X)\| \|W' - W\|\big] \\
		&\le \frac{1}{4}  \E[\|\phi(X)\|] \sqrt{\E\big[\|W'-W\|^2\big] } \\
		&= \frac{1}{4} \E[\|\phi(X)\|] \sqrt{ 2 R_2(W|Z^n) }
	\end{align}
	where we used \eqref{eq:TV_logistic} and Lemma~\ref{lm:posterior_error}.
\end{proof}
\noindent 

The upper bound in \eqref{eq:logistic_log_ub} shows that the rate of convergence of $\rm MER_{\log}$ in $n$ for logistic regression is the same as that for $R_2(W|Z^n)$, as $\log(1+u)\le u$ for $u>0$. This improves the upper bound given in \eqref{eq:MER_log_logistic_cont_Psi} when $R_2(W|Z^n)$ is small, e.g., when $n$ is large.

	\paragraph{Nonlinear and linear regression (continued)}\mbox{}
	
	\noindent
	We also continue with the discussion on the nonlinear and linear regression problems in Section~\ref{sec:MER_cont_g(x,w)}, where $Y = g(X,W) + V$, $\sW = \R^d$, $X$ and $W$ are independent, and $V \sim \mathcal{N}(0,\sigma^2)$ is independent of $(X,W)$.
	Under this model,
	\begin{align}
		\mathcal W_2^2(K_{Y|x,w'}, K_{Y|x,w}) = 2\sigma^2 D_{\rm KL}(K_{Y|x,w'} \| K_{Y|x,w}) =  (g(x,w) - g(x,w'))^2 . \label{eq:W2_nonlinreg} 
	\end{align}
	From Theorem~\ref{th:MER_log_KL}, Theorem~\ref{prop:regression_quadratic} and Lemma~\ref{lm:posterior_error}, we obtain the following upper bounds for nonlinear regression.
	\begin{corollary}\label{prop:reg_NN}
		For the above nonlinear regression problem, let $s_g \deq \E\big[ \sup_{w\in\sW} \|\nabla_w g(X,w) \|^2\big]$. Then for the log loss,
		\begin{align}
			{\rm MER}_{\log} 
			&\le \frac{1}{\sigma^2} R_2(g(X,W)|X,Z^n)
			\le \frac{s_g}{\sigma^2}  R_2(W|Z^n) , 
		\end{align}
		while for the quadratic loss,
		\begin{align}
			{\rm MER}_{2} 
			 &\le  4\sqrt{ 2 ( \E[g(X,W)^2] + \sigma^2 ) R_2(g(X,W)|X,Z^n) } \\
			 &\le  4 \sqrt{ 2 ( \E[g(X,W)^2] + \sigma^2) s_g R_2(W|Z^n) } .
		\end{align}
	\end{corollary}
	\begin{proof}
		For the log loss, 
		\begin{align}
			{\rm MER}_{\log}
			&\le \E[D_{\rm KL}(K_{Y|X,W} \| K_{Y|X,W'})] \label{eq:pf_nonlinear_reg_1} \\
			&= \frac{1}{2\sigma^2} \E\big[\big(g(X,W) - g(X,W')\big)^2\big] \label{eq:pf_nonlinear_reg_2} \\
			&= \frac{1}{\sigma^2} R_2(g(X,W)|X,Z^n) \label{eq:pf_nonlinear_reg_21} \\
			&\le  \frac{1}{2\sigma^2} \E\big[ ( \sup\nolimits_{w\in\sW} \|\nabla_w g(X,w) \|)^2 \|W-W'\|^2\big] \label{eq:pf_nonlinear_reg_3} \\
			&=  \frac{1}{\sigma^2} \E\big[ \sup\nolimits_{w\in\sW} \|\nabla_w g(X,w) \|^2\big] R_{2}(W|Z^n) \label{eq:pf_nonlinear_reg_5} .
		\end{align}
		where \eqref{eq:pf_nonlinear_reg_1} follows from Theorem~\ref{th:MER_log_KL}; \eqref{eq:pf_nonlinear_reg_2} is from \eqref{eq:W2_nonlinreg}; \eqref{eq:pf_nonlinear_reg_21} is due to Lemma~\ref{lm:posterior_error} for the quadratic loss; \eqref{eq:pf_nonlinear_reg_3} is due to \eqref{eq:pf_nonlinear_reg_2} and Lemma~\ref{lm:Lip_multi}; and \eqref{eq:pf_nonlinear_reg_5} follows from Lemma~\ref{lm:posterior_error}.
		
		For the quadratic loss, from Theorem~\ref{prop:regression_quadratic}, \eqref{eq:W2_nonlinreg} and Lemma~\ref{lm:posterior_error}, and a similar reasoning as above,
	\begin{align}
		{\rm MER}_{2} 
		&\le  4 \sqrt{ \E[Y^2] \E[ \big(g(X,W) - g(X,W')\big)^2 ] } \\
		&=  4 \sqrt{ 2 \E[Y^2] R_2(g(X,W)|X,Z^n )  } \\
		&\le  4 \sqrt{ 2 \E[Y^2]  \E\big[ \sup\nolimits_{w\in\sW} \|\nabla_w g(X,w) \|^2\big] R_{2}(W|Z^n) } \\
		&=  4 \sqrt{ 2 \big(\E[g(W,X)^2] + \sigma^2\big) s_g R_2(W|Z^n) } , 
	\end{align}	
		which proves the second upper bound.
	\end{proof}

For the special case of linear regression with Gaussian prior $P_W = \mathcal N(0,\sigma_W^2 \mathbf I_d)$, we have $g(x,w) = w^\top \phi(x)$, $s_g=\E[\|\phi(X)\|^2]$, and $R_2(W|Z^n) = \E[{\rm tr} (C_{W|Z^n})]$ with $C_{W|Z^n}$ given in \eqref{eq:C_W|Z^n}; Corollary~\ref{prop:reg_NN} in this case gives
	\begin{align}
		{\rm MER}_{\log}  
		&\le \frac{1}{\sigma^2} \E[\|\phi(X)\|^2] \E[{\rm tr} (C_{W|Z^n})]  , \label{eq:lin_reg_log_ub1}
	\end{align}
	and 
	\begin{align}
		{\rm MER}_{2} 
		&\le 4 \sqrt{ 2 \big(\sigma_W^2 \E[\|\phi(X)\|^2] + \sigma^2\big) \E[\|\phi(X)\|^2] \E[{\rm tr} (C_{W|X^n, Y^n})] } . \label{eq:lin_reg_2_ub1}
	\end{align}
From the exact expressions of MER given in \eqref{eq:MER_log_lin_reg_1} and \eqref{eq:MER_2_lin_reg_1}, we see that the upper bound for ${\rm MER}_{\log}$ in \eqref{eq:lin_reg_log_ub1} is order-optimal for vanishing MER; while the upper bound for ${\rm MER}_2$ in \eqref{eq:lin_reg_2_ub1} is not, unlike the upper bound \eqref{eq:rlz_n_lin_MER2} for ${\rm MER}_2$ given by Theorem~\ref{th:rlz_n}. In Appendix~\ref{appd:MER2_lin_mi}, we derive an alternative upper bound for $\rm MER_2$ based on Theorem~\ref{thm:regression_genloss_MI} in Section~\ref{sec:MER_gen_loss}, however it is not order-optimal either.
Nevertheless, the upper bound in \eqref{eq:lin_reg_2_ub1} can be tighter than the order-optimal upper bound in \eqref{eq:rlz_n_lin_MER2} when $R_2(W|Z^n)$ is large, e.g., when $n$ is small. 

	We also see from Theorem~\ref{th:rlz_n} and Corollary~\ref{prop:reg_NN} that the MER upper bounds for the general nonlinear regression problem depend on $n$ only through $R_2(g(X,W)|X,Z^n)$ or $R_2(W|Z^n)$. Although closed-form expressions of these quantities are generally intractable, the upper bounds explicitly show how the epistemic part of the overall prediction uncertainty depends on the model uncertainty, which can be quantified by the corresponding MMSE. 
	Moreover, the upper bounds obtained in terms of $R_2(g(X,W)|X,Z^n)$ can be much tighter than those in terms of $R_2(W|Z^n)$, 
	especially when multiple values of $W$ map to the same function $g(\cdot,W)$, e.g., when $g:\sX\times\sW\rightarrow\sY$ can be represented by over-parametrized neural networks \cite{math_nn21}.

	\section{Extensions}\label{sec:extensions}
	\subsection{Multiple model families}\label{sec:multi_model}
	Instead of being described by a single model family, in many cases the joint distribution of $X$ and $Y$ can be better represented by a finite class of model families 
	${\mathbb M} = \{\cM_m,m\in\sM\}$ all together, where each family $\cM_m = \{P_{X,Y|w,m}, w\in\sW_m\}$ is a collection of parametrized joint distributions of $(X,Y)$.
	The class of model families $\mathbb M$ is also known as the \emph{model class}, and the index $m$ of each family is also known as the \emph{model index} \cite{model_select18}.
	In the Bayesian formulation, the model index $M$ is treated as a random element of $\sM$ with prior $P_M$; given a model index $m$, the model parameters are represented as a random vector in $\sW_m$ with prior $P_{W|m}$.
	As before, denoting $Z_i\deq (X_i,Y_i)$, $i=1,\ldots,n$, as the observed data and $Z\deq (X,Y)$ as a fresh pair, the quantities under consideration are assumed to be generated from the joint distribution 
	\begin{align}\label{eq:joint_dist_multi}
		P_{M, W, Z^n, Z} = P_{M} P_{W|M} \Big(\prod\limits_{i=1}^n  P_{Z_i|W,M}\Big) P_{Z|W,M} 
	\end{align}
	where $P_{Z_i|W,M} = P_{Z|W,M}$ for $i = 1,\ldots,n$.
	In the same spirit in the single model family setting, we can define the MER in the above multi-model family setting as follows.
	\begin{definition}
		In the multi-model family setting, the fundamental limit of the Bayes risk with respect to the loss function $\ell$ is defined as
		\begin{align}\label{eq:RBW_def_multi}
			R_\ell(Y|X,W,M) = \inf_{\Psi: \sX\times\sW\times\sM \rightarrow \sA} \E[\ell(Y,\Psi(X,W,M))] .
		\end{align}
	\end{definition}
	\begin{definition}
		In the multi-model family setting, the minimum excess risk with respect to the loss function $\ell$ is defined as
		\begin{align}
			{\rm MER}_{\ell} = R_\ell(Y|X,Z^n) - R_\ell(Y|X,W,M) .
		\end{align}
	\end{definition}
	
	Similar to Lemma~\ref{lm:MER_log_mi} and Theorem~\ref{th:MER_log_KL} in the single model family setting, for the log loss, the MER in the multi-model family setting can be related to the conditional mutual information $I(M,W;Y|X,Z^n)$ and its upper bounds.
	\begin{theorem}\label{prop:ub_MER_log_multi}
		In the multi-model family setting, with the log loss,
		\begin{align}
			{\rm MER}_{\log} = I(M,W;Y|X,Z^n) ,
		\end{align}
		which can be upper-bounded by
		$
		\frac{1}{n} I(M,W; Y^n|X^n) .
		$
		Further, if $P_{X,Y|w,m}=P_{X|w,m}  K_{Y|X,w,m}$ for all $(m,w)\in\sM\times\sW_m$, then 
		\begin{align}
			{\rm MER}_{\log} \le \E[D_{\rm KL}(K_{Y|X,W,M} \| K_{Y|X,W',M'})] ,
		\end{align}
		where $(M',W')$ is a sample from the posterior distribution $P_{M,W|X,Z^n}$ such that $(M,W)$ and $(M',W')$ are conditionally i.i.d.\ given $(X,Z^n)$.
	\end{theorem}
	
	In addition, we can still bound the MER in terms of the deviation of the posterior predictive distribution from the true predictive model, similar to the results in Section~\ref{sec:MER_posterior_deviation}.
	As in the predictive modeling framework, suppose that for each model family $\mathcal M_m\in\mathbb M$, $P_{X,Y|w,m}=P_{X|w,m}  K_{Y|X,w,m}$ for all $w\in\sW_m$.
	Then for any statistical distance $D$, a diameter-like quantity of the model class $\mathbb M$ with respect to $\! D$ can be defined as
	\begin{align}
		{\rm diam}(\mathbb M,D) = \max_{ m\neq m'\in\sM} \sup_{w\in\sW_m,w'\in\sW_{m'}} \sup_{ x\in\sX} D(K_{Y|x,w',m'} , K_{Y|x,w,m}) .
	\end{align}
	\noindent With the above definition 
	we have the following general upper bound on the deviation of the posterior predictive distribution from the true predictive model.
	The proof is given in Appendix~\ref{appd:pf_prop_ub_E_div_multi}.
	\begin{theorem}\label{prop:ub_E_div_multi}
		In the multi-model setting, for any statistical distance $D$ that is convex in the first argument,
		\begin{align}\label{eq:ub_f_div_E_multi_MW}
			&\E[ D(P_{Y|X,Z^n} , K_{Y|X,W,M}) ] 
			\le  \E[D(K_{Y|X,W',M'} , K_{Y|X,W,M})] ,
		\end{align}
		where $(M',W')$ is a sample from the posterior distribution $P_{M,W|X,Z^n}$ such that $(M,W)$ and $(M',W')$ are conditionally i.i.d.\ given $(X,Z^n)$.
		The right side of \eqref{eq:ub_f_div_E_multi_MW} can be further upper-bounded by
		\begin{align}\label{eq:ub_f_div_E_multi}
			\E[D(K_{Y|X,W',M} , K_{Y|X,W,M})] + 2 {\rm diam}(\mathbb M,D) R_{01}(M|X,Z^n) ,
		\end{align}
		where 
		$W'$ is a sample from the posterior distribution $P_{W|X,Z^n,M}$ such that $W'$ and $W$ are conditionally i.i.d.\ given $(X,Z^n,M)$.
		If $D$ is convex in the second argument, we obtain another set of upper bounds by exchanging the order of the arguments of $D$ in the results above.
	\end{theorem}
	\noindent 
	
	Theorem~\ref{prop:ub_E_div_multi} shows that under the multi-model family setting, the expected deviation consists of two parts: the first part can be related to the estimation error of the model parameters when the model index is correctly identified, which depends on the complexity of each model family; the second part is related to the penalty when the model index is wrongly identified, which depends on the overall complexity of the model class and the error probability of model index estimation.

	As an example, for linear regression with multiple model families, the predictive model in the $m$th family can be described as $K_{Y|x,w,m} = \mathcal N(w^\top \phi(x,m), \sigma^2)$, where $w\in\sW_m \subset \R^{d_m}$ is the model parameter vector and $\phi(x,m)\in\R^{d_m}$ is the feature vector of the observation $x$. We also assume that $X$ is independent of $(M,W)$.
	In this case,
	\begin{align}
		D_{\rm KL}(K_{Y|x,w',m'} \| K_{Y|x,w,m}) &= \frac{1}{2\sigma^2} \big(w'^\top\phi(x,m') - w^\top\phi(x,m)\big)^2 , \label{eq:KL_linreg_multi} 
	\end{align}
	and 
	\begin{align}
		{\rm diam}(\mathbb M, D_{\rm KL})  = \frac{1}{2\sigma^2}\max_{m\neq m' \in\sM} \sup_{w\in\sW_m,w'\in\sW_{m'}} \sup_{ x\in\sX}  
		\big( w'^\top \phi(x,m') - w^\top \phi(x,m)  \big)^2 .
	\end{align}
	From Theorem~\ref{prop:ub_MER_log_multi}, the chain rule of mutual information, Theorem~\ref{prop:ub_E_div_multi}, and the previous results on linear regression in the single model family, we have the following upper bounds for ${\rm MER}_{\log}$ for linear regression with multiple models:
		\begin{align}
			{\rm MER}_{\log}
			&\le \frac{1}{2\sigma^2}\E[\|\phi_M(X)\|^2] R_2(W|Z^n,M) + H(M|Z^n)  
		\end{align}
		and
		\begin{align}
			{\rm MER}_{\log} \le   \frac{1}{\sigma^2}\E[\|\phi_M(X)\|^2] R_2(W|Z^n,M) + 2 {\rm diam}(\mathbb M, D_{\rm KL}) R_{01}(M|Z^n) 
		\end{align}
		where
		\begin{align}
			R_2(W|Z^n,M) = \sum_{m\in\sM} P_M(m) \E[{\rm tr}(C_{W|Z^n,m}) ]
		\end{align}
		with
		$
		C_{W|Z^n,m} = ({\sigma_{W,m}^{-2}}\mathbf{I}_{d_m} + {\sigma^{-2}} \mathbf\Phi_m \mathbf\Phi_m^\top )^{-1}
		$
		and $\mathbf\Phi_m = [\phi(X_1,m),\ldots,\phi(X_n,m)]$ being the $d_m\times n$ feature matrix for the $m$th model family.
We see that the MER consists of a part that depends on the minimum achievable model parameter estimation error given each model index, and a part that depends on the uncertainty of model index estimation and the ``diameter'' of the model class.

	\subsection{MER in nonparametric models}\label{sec:nonpara}      
	The definition of MER can also be extended to Bayesian learning under a nonparametric predictive model that can be specified in terms of a random process.
	Formally, consider the case where $F$ is a real-valued random process indexed by $x\in\sX$, and the predictive model is a probability transition kernel $K_{Y|F(X)}$. It is further assumed that $F$ is \textit{a priori} independent of $X$. 
	The observed data and the fresh pair are assumed to be generated from the joint distribution 
	\begin{align}\label{eq:joint_dist_nonpara}
		P_{F, Z^n, Z} = P_{F} \Big(\prod\limits_{i=1}^n  P_{Z_i|F}\Big) P_{Z|F} 
	\end{align}
	where $P_{Z_i|F} = P_{Z|F} = P_X  K_{Y|F(X)}$ for $i = 1,\ldots,n$.
	Two simple examples of the above model are 1) noiseless Gaussian process regression model, where $F$ is a Gaussian process with a mean function $m:\sX\rightarrow\R$ and a covariance function $k:\sX\times\sX\rightarrow\R$, and $Y=F(X)$; and 2) binary classification model with Gaussian process as a latent function \cite{GP_book}, where $F$ can be the same Gaussian process, and $K_{Y|F(X)}(1|f(x))=\sigma(f(x))$ with $\sigma(\cdot)$ being the logistic sigmoid function.
	
	In the same spirit in the parametric case, the MER under the above nonparametric model can be defined as
	\begin{align}\label{eq:def_MER_nonpara}
		{\rm MER}_\ell = R_\ell(Y|X,Z^n) - R_\ell(Y|F(X)) ,
	\end{align}
	where $R_\ell(Y|X,Z^n)$ and $R_\ell(Y|F(X))$ are defined according to the general definition of the Bayes risk in \eqref{eq:def_Brisk}, and correspond to \eqref{eq:RB_def} and \eqref{eq:RBW_def} respectively.
	For the log loss, using the fact that $H_\mu(Y|F(X),Z^n) = H_\mu(Y|F(X))$ and following the same argument as in Corollary~\ref{co:classification_log_mi_total}, we have
	\begin{align}
		{\rm MER}_{\rm log} &= I(F(X) ; Y|X,Z^n) \le \frac{1}{n} I(F(X);Y^n|X^n) .
	\end{align}
	For the quadratic loss,
	\begin{align}
		{\rm MER}_{2} &=  R_2(Y|X,Z^n) - R_2(Y|F(X)) .
	\end{align}
	In the special case of noiseless Gaussian process regression model, $R_2(Y|F(X))=0$, which implies
	\begin{align}
		{\rm MER}_{2} &=  R_2(F(X)|F(X_1),\ldots,F(X_n)) \\
		&= \E\big[k(X,X) - k(X,X^n)^\top \Sigma^{-1}(X^n) k(X,X^n)\big] 
	\end{align}
	where $k(X,X^n)$ is the covariance vector between $F(X)$ and $(F(X_1),\ldots,F(X_n))$, and $\Sigma(X^n)$ is the covariance matrix of $(F(X_1),\ldots,F(X_n))$.
	The above expression may be further analyzed using the eigenfunction expansion of the covariance function $k$ \cite{GP_book}.
	For the binary classification model with Gaussian process as the latent function, or more general models specified with non-Gaussian processes, the MER may not have a simple close-form expression.

	\section{Discussion}\label{sec:discuss}
	We have defined the minimum excess risk in Bayesian learning with respect to general loss functions, and presented general methods for obtaining upper bounds for this quantity. 
	How to lower-bound this quantity is left as an open problem.
	We would like to close the paper by discussing the following two aspects.
	
	\subsection{Tightness and utility of the results}
	Two methods for deriving upper bounds on the MER have been presented: one method relates the MER to $I(W;Y|X,Z^n)$; the other one relates it to $R_2(W|X,Z^n)$.
	
	With the precise asymptotic expansion of $I(W;Z^n)$, the first method is suitable for asymptotic analysis for a wide range of loss functions. Using this method, we have shown that for any bounded loss function, the MER scales with the data size $n$ as $O(\sqrt{1/n})$ in general, while for the log loss, the squared loss with bounded $Y$, and bounded loss under realizable models, this convergence rate can be improved to $O(1/n)$. When the model parameter lies in a compact subset of $\R^d$, or when the VC dimension of the generative function class is $d$, the MER bounds can also capture the dependence on $d$, as $O(\sqrt{d/n})$ or $O(d/n)$ in different settings. An MER lower bound of $\Omega(d/n)$ is derived in a follow-up work \cite[Theorem~10]{RD_MER21} for the cases where the excess risk of using $\Psi^*(X,W')$ as the plug-in decision rule is lower bounded by $\|W-W'\|^2$, which matches upper bounds in certain settings derived in this work. 
	
	The second method has the potential to provide us with nonasymptotic upper bounds.
	The only explicit expression for $R_2(W|X,Z^n)$ we have so far is for linear regression, for which we have derived order-optimal upper bound for both ${\rm MER}_{\log}$ and ${\rm MER}_2$.
	In order to obtain explicit upper bounds for problems beyond linear regression, e.g.\ logistic regression or nonlinear regression, we would need upper bounds on $R_2(W|X,Z^n)$, or other forms of minimum model parameter estimation error in these settings.
	Nevertheless, from the examples on logistic regression, linear regression, and nonlinear regression, we see that the MER upper bounds obtained from the second method depend on $n$ only through $R_2(W|X,Z^n)$. This explicitly shows how the model uncertainty translates to the epistemic uncertainty and contributes to the overall prediction uncertainty.
	The definition of MER provides such a principled way to define different notions of uncertainties in Bayesian learning,
	and its study guides the analysis and estimation of these uncertainties, which is an increasingly important direction of research with wide applications.

	\subsection{MER in Bayesian learning vs.\ excess risk in frequentist learning}
	The distinguishing feature of Bayesian learning is that the generative model of data is assumed to be drawn from a known model family according to some prior distribution. As a result, the MER in Bayesian learning is determined by how accurate the model can be estimated, and there is no notion of \emph{approximation error} unless the model family or the prior distribution is misspecified. This stands in contrast to the frequentist formulation of statistical learning where the data-generating model is assumed to be completely unknown, and the excess risk consists of an estimation error part and an approximation error part.
	
	In the discussion on the setting with multiple model families in Section~\ref{sec:multi_model}, it is shown that the MER there not only depends on the accuracy of the model parameter estimation within a fixed model, but also on the product of the error probability of model index estimation and a diameter-like term measuring the largest statistical distance among the predictive models.
	The latter quantity may be viewed as an analogue of the approximation error in the frequentist setting, as it
	upper-bounds the penalty incurred by a wrong estimate of the model index. Its impact
	on the MER vanishes as the data size increases though, as the error probability of model index estimation would eventually go to zero.
	
	Another connection to the frequentist learning would be the identical expressions shared by the MER-information relationship in Corollary~\ref{co:MER_mi_Z_W} and the generalization-information relationship in the frequentist setting of \cite[Theorem~1]{XuRaginsky17}. It shows the important roles played by information-theoretic quantities in the theory of statistical learning.

	\section*{Appendix}
	\appendix

	\section{Miscellaneous lemmas}
	\subsection{Regularity conditions for \eqref{eq:mi_expan}}\label{appd:cond_I(W;Z^n)}
	The regularity conditions for \eqref{eq:mi_expan} to hold are listed here for completeness. These conditions are drawn from  in \cite[Section~2]{Cla_Bar94}.
	Let $\sW \subset \R^d$ and assume that the densities of $P_{Z|w}$ exist with respect to the Lebesgue measure for all $w\in\sW$. Also assume the parameter space $\sW$ has a non-void interior and its boundary has a $d$-dimensional Lebesgue measure zero.
	
	\begin{enumerate}
		\item  The density $p_{Z|W}(z|w)$ is twice continuously differentiable in $w$ for almost every $z$. There exists $\delta(w)$ such that for every $j, k \in \{1, \dots, d\}$:
		\begin{align}
			\E \Big[\sup_{w': \|w' - w\|\leq \delta(w)}\Big| \frac{\partial^2}{\partial w_j' \partial w_k'} \log p_{Z|W}(Z|w')\Big|\Big]
		\end{align}
		is finite and continuous in $w$.
		In addition, for each $j \in \{1, \dots, d\}$:
		\begin{align}
			\E \Big[\Big| \frac{\partial}{\partial w_j} \log p_{Z|W}(Z|w)\Big|^{2+\zeta}\Big]
		\end{align}
		is finite and continuous, as a function of $w$, for some $\zeta>0$.
		
		\item Conditions on Fisher information matrix: define the matrix
		\begin{align}
			[I(w)]_{j, k} = \E \Big[\frac{\partial}{\partial w_j} \log p_{Z|W}(Z|w) \frac{\partial}{\partial w_k} \log p_{Z|W}(Z|w)\Big],
		\end{align}
		and 
		\begin{align}
			[J(w)]_{j, k} = \Big[\frac{\partial^2}{\partial w'_j\partial w'_k} D_{\rm{KL}} \Big(P_{Z|w} \| P_{Z|w'}\Big)\Big|_{w' = w}\Big];
		\end{align}	
		we have $I(w) = J(w)$. The matrix $I(w)$ is also positive definite.
		
		\item For $w\neq w'$, we have $P_{Z|w} \neq P_{Z|w'}$.
		
		\item The prior on $W$ is continuous and is supported on a compact subset of the interior of $\sW$.
	\end{enumerate}
	
	\subsection{A transportation inequality}\label{appd:pf_prop_genloss_MI}
	The following lemma is adapted from \cite[Lemma~4.18]{concen_ineq_BLM} and \cite[Theorem~2]{JiaoHanWeissman_bias_17}.

	\begin{lemma}\label{lm:DPQ_variational_gen}
		For distributions $P$ and $Q$ on a set $\sU$ and a function $f:\sU\rightarrow\R$, suppose there exists a function $\varphi$ over $(0,b)$ with some $b \in (0, \infty]$ such that
		\begin{align}
			\log \E_Q\big[e^{-\lambda (f(U) - \E_Q f(U) )}\big] \le \varphi(\lambda), \quad\forall\, 0< \lambda < b . \label{eq:DPQ_var_cond-}
		\end{align}
		Then
		\begin{align}\label{eq:DPQ_variational_gen_-}
			\E_Q[f(U)] - \E_P[f(U)] \le \varphi^{*-1}(D_{\rm KL}(P\| Q)) ,
		\end{align}
		where 
		\begin{align}\label{eq:Legendre_dual_def}
			\varphi^*(\gamma) &= \sup_{0< \lambda < b} \lambda \gamma - \varphi(\lambda) , \quad \gamma\in\R 
		\end{align}
		is the Legendre dual of $\varphi$ 
		and $\varphi^{*-1}$ is the inverse of $\varphi^*$, defined as
		\begin{align}
			\varphi^{*-1}(x) &= \sup\{\gamma\in\R:\varphi^*(\gamma)\le x\} , \quad x\in\R \label{eq:inv_Legendre_dual} .
		\end{align}
	\end{lemma}
	\begin{proof}[Proof of Lemma~\ref{lm:DPQ_variational_gen}]
		The Donsker-Varadhan theorem states that
		\begin{align}
			D_{\rm KL}(P \| Q) = \sup_{g:\sU\rightarrow\R} \E_P[g(U)] - \log \E_Q[e^{g(U)}] ,
		\end{align}
		which implies that
		\begin{align}
			D_{\rm KL}(P \| Q) &\ge \sup_{0<\lambda<b} \lambda(\E_Q[f(U)] - \E_P[f(U)]) - \log \E_Q[e^{-\lambda(f(U)-\E_Q f(U))}] \\
			&\ge \sup_{0<\lambda<b}\lambda(\E_Q[f(U)] - \E_P[f(U)]) - \varphi(\lambda) \\
			&= \varphi^*(\E_Q[f(U)] - \E_P[f(U)]) .
		\end{align}
		Consequently, from the definition in \eqref{eq:inv_Legendre_dual},
		\begin{align}
			\E_Q[f(U)] - \E_P[f(U)] \le \varphi^{*-1}(D_{\rm KL}(P \| Q)) ,
		\end{align}
		which proves \eqref{eq:DPQ_variational_gen_-}.
	\end{proof}

	\subsection{Series with growth rate $\log n$}\label{appd:lm_sum_log_n}
	The following lemma is a restatement of \cite[Lemma~6]{Haussler_Opper_MI95}.
	\begin{lemma}\label{lm:sum_log_n}
	Suppose $(a_1,a_2,\ldots)$ and $(b_1,b_2,\ldots)$ are two sequences of real numbers such that $a_n = \sum_{i=1}^n b_i$ for all $n$. 
	Then
	\begin{align}
	\lim_{n\rightarrow\infty}\frac{a_n}{\log n} = \lim_{n\rightarrow\infty}n b_n ,
	\end{align}
	whenever both limits exist.
	\end{lemma}

	\subsection{Convexity of $\mathcal W_p^p(P,Q)$}\label{appd:cvx_W_p^p}
	\begin{lemma}\label{lm:cvx_W_p^p}
		The $p$th power of the $p$-Wasserstein distance is jointly convex in its two arguments, i.e. $\mathcal W_p^p(P,Q)$ is convex in $(P,Q)$.
	\end{lemma}
	\begin{proof}
		By definition,
		\begin{align}
			\mathcal W_p^p(P,Q) = \inf_{\Pi(P,Q)} \E_{(X,Y)\sim\Pi} [\|X-Y\|^p] .
		\end{align}
		For arbitrary $(P_1,Q_1)$, $(P_2,Q_2)$, and $\gamma\in[0,1]$, let $\Pi_1$ and $\Pi_2$ be the optimal couplings for $\mathcal W_p^p(P_1,Q_1)$ and $\mathcal W_p^p(P_2,Q_2)$ respectively. Then
		\begin{align}
			& \mathcal W_p^p(\gamma P_1 + (1-\gamma) P_2,\gamma Q_1 + (1-\gamma) Q_2) \\
			\le & \E_{(X,Y)\sim \gamma\Pi_1 + (1-\gamma)\Pi_2} [\|X-Y\|^p] \\
			= & \gamma \E_{(X,Y)\sim \Pi_1 } [\|X-Y\|^p] + (1-\gamma) \E_{(X,Y)\sim \Pi_2 } [\|X-Y\|^p] \\
			= & \gamma \mathcal W_p^p(P_1,Q_1) + (1-\gamma) \mathcal W_p^p(P_2,Q_2) ,
		\end{align}
		where the first inequality is because $\gamma\Pi_1 + (1-\gamma)\Pi_2$ is a coupling of $\gamma P_1 + (1-\gamma) P_2$ and $\gamma Q_1 + (1-\gamma) Q_2$.
		This shows the convexity of $\mathcal W_p^p(P,Q)$ in $(P,Q)$.
	\end{proof}
	
	
	\subsection{Lipschitz continuity of multivariate function}\label{appd:Lip_multi}
	The following lemma states a sufficient condition for a multivariate function to be Lipschitz continuous \cite{Waker_Lip_f}.
	\begin{lemma}\label{lm:Lip_multi}
		Suppose a function $f:\R^n \rightarrow \R$ is continuously differentiable everywhere in a convex set $\sX \subset \R^n$.
		If $c>0$ is such that $\|\nabla f(x)\| \le c$ for all $x\in \sX$, then $|f(y) - f(x)| \le c \|y-x\|$ for all $x,y\in \sX$.
	\end{lemma}

	\subsection{Proof of Lemma~\ref{lm:H_cont_miny_TV}}\label{appd:pf_H_cont_miny_TV}
	With the log loss, the generalized entropy of discrete $Y$ is the Shannon entropy. We have
	\begin{align}
		H(P) - H(Q) 
		&= \E_P[-\log{P(U)}] - \E_Q[-\log{Q(U)}] \\
		&\le \E_P[-\log{Q(U)}] - \E_Q[-\log{Q(U)}] \\
		&= \sum_{u\in\sU} (P(u)-Q(u)) (-\log{Q(u)}) \\
		&\le (-\log\min_{u\in\sU}Q(u)) d_{\rm TV}(P,Q) ,
	\end{align}
	where the first inequality follows from the fact that $H(P)=\inf_{Q} \E_P[-\log Q(U)]$, and the last inequality follows from the fact that $-\log Q(u)\in [0,-\log \min_{u\in\sU}Q(u)]$ and the dual representation of the total variation distance.
	
	For the zero-one loss, the generalized entropy of discrete $Y$ is one minus the maximum probability. We have
	\begin{align}
		&(1- \max\nolimits_{u\in\sY}P(u)) - (1-\max\nolimits_{u\in\sU}Q(u)) \nonumber  \\
		=& \max\nolimits_{u\in\sY}Q(u)- \max\nolimits_{u\in\sY}P(u) \\
		\le&  Q(u_{\max}) - P(u_{\max}) \label{eq:pf_cont_R01_0} \\
		\le&  d_{\rm TV}(Q, P ) \label{eq:pf_cont_R01_1} 
	\end{align}
	where in \eqref{eq:pf_cont_R01_0}, $u_{\max} \deq \argmax_{u\in\sY}Q(u)$; \eqref{eq:pf_cont_R01_1} follows from the fact that $d_{\rm TV}(Q,P) = \sup\nolimits_{E\subset\sU} Q[E] - P[E]$ for any pair of distributions on $\sU$. The claim follows from the fact that the total variation distance is symmetric.

	\subsection{Proof of Lemma~\ref{lm:Var_cont}}\label{appd:pf_lm_Var_cont}
	First note that according to the definition of the $\mathcal W_2$ distance, $\E[U^2] = \mathcal W_2^2 (P_U,\delta_0)$ and $\E[V^2] = \mathcal W_2^2 (P_V,\delta_0)$, where $\delta_0$ denotes the point mass at $0$. Then
	\begin{align}
		\Var[U] - \Var[V] 
		&= \E[U^2] - \E[V^2] + (\E[U] + \E[V])(\E[V] - \E[U]) \\
		&\le ( \mathcal W_2^2 (P_U,\delta_0) -  \mathcal W_2^2 (P_V,\delta_0)) + |\E[U] + \E[V]| \mathcal W_1(P_U, P_V) \label{eq:Var_cont_pf1} \\
		&\le ( \mathcal W_2 (P_U,\delta_0) +  \mathcal W_2 (P_V,\delta_0))| \mathcal W_2 (P_U,\delta_0) -  \mathcal W_2 (P_V,\delta_0)| + \nonumber \\
		& \quad\,\,\, |\E[U] + \E[V]| \mathcal W_2(P_U, P_V) \label{eq:Var_cont_pf2} \\
		&\le (\sqrt{\E[U^2]} + \sqrt{\E[V^2]}) \mathcal W_2 (P_U,P_V) + |\E[U] + \E[V]| \mathcal W_2(P_U, P_V) \label{eq:Var_cont_pf3} \\
		&\le 2(\sqrt{\E[U^2]} + \sqrt{\E[V^2]}) \mathcal W_2 (P_U,P_V) \label{eq:Var_cont_pf4} 
	\end{align}
	where we have used the triangle inequality for the $\mathcal W_2$ distance and the fact that 
	\begin{align}
		|\E[U] - \E[V]| \le \mathcal W_1(P_U, P_V) \le \mathcal W_2(P_U,P_V) .
	\end{align}

	\section{MER upper bound for linear regression based on Theorem~\ref{thm:regression_genloss_MI}}\label{appd:MER2_lin_mi}
	To make use of Theorem~\ref{thm:regression_genloss_MI} for linear regression with the quadratic loss, let $Y=W^\top \phi(X) + V$ where $V\sim \mathcal N(0,\sigma^2)$, and assume $X$ is independent of $W$. In addition, let $W'$ be sampled from $P_{W|Z^n}$ independently of everything else. 
	Since $\psi^*(X,W') = W'^\top \phi(X)$, we have
	\begin{align}
		(Y - \psi^*(X,W'))^2 &= (Y - W'^\top \phi(X))^2 
		= ((W-W')^\top \phi(X) + V)^2 .
	\end{align}
	Since $W'$ is a conditionally i.i.d.\ copy of $W$ given $(X,Z^n)$, it can be seen that the conditional distribution of $(W-W')^\top \phi(X)$ given $(X,Z^n)=(x,z^n)$ is Gaussian with zero mean and variance $2\phi(x)^\top C_{W|z^n} \phi(x)$.
	It follows that conditional on $(X,Z^n)=(x,z^n)$, $(Y - \psi^*(x,W'))^2 $ has the same distribution as $\big(2\phi(x)^\top C_{W|z^n} \phi(x) + \sigma^2\big) U^2$, where $U\sim \mathcal N(0,1)$.
	As a consequence of the fact that
	\begin{align}
		\log \E[e^{-\lambda( \sigma_\chi^2 U^2 - \E[\sigma_\chi^2 U^2])}] 
		= \lambda\sigma_\chi^2 - \frac{1}{2} \log(1+2\sigma_\chi^2\lambda)
		\le \sigma_\chi^4 \lambda^2 \deq \varphi(\lambda) \quad \text{for $\lambda > 0$},
	\end{align}
	the fact that $\varphi^{*-1}(\gamma) = 2\sigma_\chi^2 \sqrt{\gamma}$,
	the assumption that $\sup_{x,x^n} \phi(x)^\top C_{W|z^n} \phi(x) \le b$, 
	the fact that $I(W;Y|X,Z^n) = \E[\frac{1}{2} \log (1 + {\phi(X)^\top C_{W|Z^n}\phi(X)}/{\sigma^2})]$,
	and Theorem~\ref{thm:regression_genloss_MI}, we have
	\begin{align}
		{\rm MER}_{2} 
		&\le 2(2b+\sigma^2) \sqrt{\frac{1}{2} \log \Big(1+\frac{1}{\sigma^2} \E\big[\phi(X) C_{W|Z^n} \phi(X)\big]\Big)} \\
		&\le 2(2b+\sigma^2) \sqrt{\frac{1}{2} \log \Big(1+\frac{1}{\sigma^2} \E[\|\phi(X)\|^2] \E[{\rm tr} (C_{W|X^n, Y^n})] \Big)} .
	\end{align}


	\section{Proof of Theorem~\ref{prop:ub_E_div_multi}}\label{appd:pf_prop_ub_E_div_multi}

	From the fact that $P_{Y|x,z^n}=\sum_{m'\in\sM}P_{M|X,Z^n}(m'|x,z^n)  \int_{\sW_{m'}} \!\!  P_{W|M,X,Z^n}({\rm d}w'|m',x,z^n) K_{Y|x,w',m'}$ and the convexity assumption of the statistical distance under consideration, the proof of the first inequality essentially follows the same steps of the proof of Lemma~\ref{lm:ub_f_div}.
	
	The second inequality is based on the first one, and can be shown as
	\begin{align}
		& \quad\,\, \E\big[ D(P_{Y|X,Z^n} , K_{Y|X,W,M}) \big]  \\
		&\le \E\big[ D(K_{Y|X,W',M'} , K_{Y|X,W,M}) \big]  \\
		&= \sum_{m\in\sM} P_M(m) \int_{\sW_m}\!\! P_{W|M}({\rm d}w|m)\int_{\sX\times\sZ^n} P_{X,Z^n|W,M}({\rm d}x,{\rm d}z^n|w,m) \, \cdot \nonumber \\ 
		& \quad\, \sum_{m'\in\sM} P_{M|X,Z^n}(m'|x,z^n) \int_{\sW_{m'}} \!\! P_{W|X,Z^n,M}({\rm d}w'|x,z^n,m') D(K_{Y|x,w',m'} , K_{Y|x,w,m}) \\
		&= S_1 + S_2
	\end{align}
	where the last step is to split the summation over $m'$ such that
	\begin{align}
		S_1 &= \sum_{m\in\sM} P_M(m) \int_{\sW_m}P_{W|M}({\rm d}w|m)\int_{\sX\times\sZ^{n}} P_{X,Z^n|W,M}({\rm d}x,{\rm d}z^n|w,m) \nonumber \\ 
		& \qquad\quad\,\, P_{M|X,Z^n}(m|x,z^n) \int_{\sW_m} \!\! P_{W|X,Z^n,M}({\rm d}w'|x,z^n,m) D(K_{Y|x,w',m} , K_{Y|x,w,m}) \\
		&\le \sum_{m\in\sM} P_M(m) \int_{\sW_m} \! P_{W|M}({\rm d}w|m)\int_{\sX\times\sZ^{n}} P_{X,Z^n|W,M}({\rm d}x,{\rm d}z^n|w,m)  \nonumber \\ 
		&\qquad\qquad\qquad\,\,  \int_{\sW_m} \! P_{W|X,Z^n,M}({\rm d}w'|x,z^n,m) D(K_{Y|x,w',m} , K_{Y|x,w,m}) \\
		&= \E[D(K_{Y|X,W',M} , K_{Y|X,W,M})]
	\end{align}
	and
	\begin{align}
		S_2 &= \sum_{m\in\sM} P_M(m) \int_{\sW_m}P_{W|M}({\rm d}w|m)\int_{\sX\times\sZ^{n}} P_{X,Z^n|W,M}({\rm d }x, {\rm d}z^n|w,m) \,\cdot \nonumber \\ 
		& \quad \sum_{m' \neq m} P_{M|X,Z^n}(m'|x,z^n) \int_{\sW_{m'}} \!\!\! P_{W|X,Z^n,M}({\rm d}w'|x,z^n,m') D(K_{Y|x,w',m'} , K_{Y|x,w,m}) \\
		&\le \Big( \max_{m,m'\in\sM, m\neq m'} \sup_{w\in\sW_m,w'\in\sW_{m'}} \sup_{x\in\sX} D(K_{Y|x,w',m'} , K_{Y|x,w,m}) \Big)  \nonumber \\
		& \quad\, \sum_{m\in\sM} P_M(m) \int_{\sW_m}P_{W|M}({\rm d}w|m) \int_{\sX,\sZ^n} P_{X,Z^n|W,M}({\rm d}x, {\rm d}z^n|w,m) 
		\sum_{m' \neq m} P_{M|X,Z^n}(m'|x,z^n) \\
		&= {\rm diam}(\mathbb M,D) \PP[M'\neq M]  \\
		&\le 2 {\rm diam}(\mathbb M,D) R_{\text{01}}(M|X,Z^n)
	\end{align}
	where the last step follows from Lemma~\ref{lm:posterior_error} applied to the zero-one loss.

	\section*{Acknowledgment}
	The authors would like to thank Yihong Wu for insightful comments on an early draft of this work; Lemma~\ref{lm:Var_cont} is given by him.
	The authors are also thankful to Max Welling and Auke Wiggers for discussions on different notions of uncertainties in Bayesian learning.
	M.~Raginsky was supported in part by the NSF CAREER award CCF-1254041, in part by the Illinois Institute for Data Science and Dynamical Systems (iDS2), an NSF HDR TRIPODS institute, under award CCF-1934986, in part by DARPA under the LwLL (Learning with Less Labels) program, and in part by ARO MURI grant W911NF-15-1-0479 (Adaptive Exploitation of Non-Commutative Multimodal Information Structure).

	\bibliography{aolin}

\begin{thebibliography}{10}
\providecommand{\url}[1]{#1}
\csname url@samestyle\endcsname
\providecommand{\newblock}{\relax}
\providecommand{\bibinfo}[2]{#2}
\providecommand{\BIBentrySTDinterwordspacing}{\spaceskip=0pt\relax}
\providecommand{\BIBentryALTinterwordstretchfactor}{4}
\providecommand{\BIBentryALTinterwordspacing}{\spaceskip=\fontdimen2\font plus
\BIBentryALTinterwordstretchfactor\fontdimen3\font minus
  \fontdimen4\font\relax}
\providecommand{\BIBforeignlanguage}[2]{{%
\expandafter\ifx\csname l@#1\endcsname\relax
\typeout{** WARNING: IEEEtran.bst: No hyphenation pattern has been}%
\typeout{** loaded for the language `#1'. Using the pattern for}%
\typeout{** the default language instead.}%
\else
\language=\csname l@#1\endcsname
\fi
#2}}
\providecommand{\BIBdecl}{\relax}
\BIBdecl

\bibitem{BishopBook06}
C.~M. Bishop, \emph{Pattern Recognition and Machine Learning (Information
  Science and Statistics)}.\hskip 1em plus 0.5em minus 0.4em\relax Berlin,
  Heidelberg: Springer-Verlag, 2006.

\bibitem{MCMC_handbook}
S.~Brooks, A.~Gelman, G.~Jones, and X.-L. Meng, \emph{Handbook of {Markov Chain
  Monte Carlo}}.\hskip 1em plus 0.5em minus 0.4em\relax Chapman Hall/CRC, 2011.

\bibitem{viBlei17}
D.~M. Blei, A.~Kucukelbir, and J.~D. McAuliffe, ``Variational inference: A
  review for statisticians,'' \emph{Journal of the American Statistical
  Association}, vol. 112, no. 518, pp. 859--877, 2017.

\bibitem{Neal_BNN}
R.~M. Neal, \emph{Bayesian Learning for Neural Networks}.\hskip 1em plus 0.5em
  minus 0.4em\relax Berlin, Heidelberg: Springer-Verlag, 1996.

\bibitem{Blundell15_w_uncert}
C.~Blundell, J.~Cornebise, K.~Kavukcuoglu, and D.~Wierstra, ``Weight
  uncertainty in neural networks,'' in \emph{International Conference on
  Machine Learning}, 2015.

\bibitem{Gal2016Dropout}
Y.~Gal and Z.~Ghahramani, ``Dropout as a {B}ayesian approximation: Representing
  model uncertainty in deep learning,'' in \emph{International Conference on
  Machine Learning}, 2016.

\bibitem{GP_book}
C.~E. Rasmussen and C.~K.~I. Williams, \emph{Gaussian Processes for Machine
  Learning (Adaptive Computation and Machine Learning)}.\hskip 1em plus 0.5em
  minus 0.4em\relax The MIT Press, 2006.

\bibitem{Bcompres17}
C.~Louizos, K.~Ullrich, and M.~Welling, ``Bayesian compression for deep
  learning,'' in \emph{Conference on Neural Information Processing Systems},
  2017.

\bibitem{NIPS2017_7141}
A.~Kendall and Y.~Gal, ``What uncertainties do we need in {B}ayesian deep
  learning for computer vision?'' in \emph{Conference on Neural Information
  Processing Systems}, 2017.

\bibitem{Depeweg2018Decomp}
S.~Depeweg, J.~M. Hern{\'a}ndez-Lobato, F.~Doshi-Velez, and S.~Udluft,
  ``Decomposition of uncertainty in {B}ayesian deep learning for efficient and
  risk-sensitive learning,'' in \emph{ICML}, 2018.

\bibitem{Hllermeier2019Aleatoric}
E.~H{\"u}llermeier and W.~Waegeman, ``Aleatoric and epistemic uncertainty in
  machine learning: A tutorial introduction,'' \emph{ArXiv 1910.09457}, 2019.

\bibitem{XuRaginsky17}
A.~Xu and M.~Raginsky, ``Information-theoretic analysis of generalization
  capability of learning algorithms,'' in \emph{Conference on Neural
  Information Processing Systems}, 2017.

\bibitem{Davisson1973}
L.~D. Davisson, ``Universal noiseless coding,'' \emph{IEEE Transactions on
  Information Theory}, vol.~19, pp. 783--795, 1973.

\bibitem{Merhav98}
N.~{Merhav} and M.~{Feder}, ``Universal prediction,'' \emph{IEEE Transactions
  on Information Theory}, vol.~44, no.~6, pp. 2124--2147, 1998.

\bibitem{haussler97}
D.~Haussler and M.~Opper, ``Mutual information, metric entropy and cumulative
  relative entropy risk,'' \emph{Ann. Statist.}, vol.~25, no.~6, pp.
  2451--2492, 12 1997.

\bibitem{Baxter1997ABT}
J.~Baxter, ``A {B}ayesian/information theoretic model of learning to learn via
  multiple task sampling,'' \emph{Machine Learning}, vol.~28, pp. 7--39, 1997.

\bibitem{HausslerVC94}
D.~Haussler, M.~Kearns, and R.~E. Schapire, ``Bounds on the sample complexity
  of {B}ayesian learning using information theory and the {VC} dimension,''
  \emph{Machine Learning}, vol.~14, no.~1, 1994.

\bibitem{LeCamYang2000}
L.~Le~Cam and G.~L. Yang, \emph{Asymptotics in Statistics Some Basic Concepts},
  2nd~ed.\hskip 1em plus 0.5em minus 0.4em\relax Springer New York, 2000.

\bibitem{ghosal2000}
S.~Ghosal, J.~K. Ghosh, and A.~W. van~der Vaart, ``Convergence rates of
  posterior distributions,'' \emph{Ann. Statist.}, vol.~28, no.~2, pp.
  500--531, 04 2000.

\bibitem{ghosal2007}
S.~Ghosal and A.~van~der Vaart, ``Convergence rates of posterior distributions
  for noniid observations,'' \emph{Ann. Statist.}, vol.~35, no.~1, pp.
  192--223, 02 2007.

\bibitem{NIPS2018_7372}
N.~G. Polson and V.~Ro\v{c}kov\'{a}, ``Posterior concentration for sparse deep
  learning,'' in \emph{Conference on Neural Information Processing Systems},
  2018.

\bibitem{Barron98}
A.~R. Barron, ``Information-theoretic characterization of {B}ayes performance
  and the choice of priors in parametric and nonparametric problems,'' in
  \emph{Bayesian Statistics 6}.\hskip 1em plus 0.5em minus 0.4em\relax Oxford
  University Press, 1998.

\bibitem{DM_PACB03}
D.~McAllester, ``{PAC-Bayesian} stochastic model selection,'' \emph{Machine
  Learning}, vol.~51, no.~1, 2003.

\bibitem{STW_pac_Bayes97}
J.~Shawe-Taylor and R.~C. Williamson, ``A {PAC} analysis of a {B}ayesian
  estimator,'' in \emph{Conference on Computational Learning Theory}, 1997.

\bibitem{Zhang_it_est06}
T.~Zhang, ``Information-theoretic upper and lower bounds for statistical
  estimation,'' \emph{IEEE Trans. Inform. Theory}, vol.~52, no.~4, pp. 1307 --
  1321, 2006.

\bibitem{BayeMixture03}
R.~Meir and T.~Zhang, ``Generalization error bounds for {B}ayesian mixture
  algorithms,'' \emph{Journal of Machine Learning Research}, 2003.

\bibitem{gunwald2004}
P.~D. Gr{\"u}nwald and A.~P. Dawid, ``Game theory, maximum entropy, minimum
  discrepancy and robust {B}ayesian decision theory,'' \emph{Ann. Statist.},
  vol.~32, no.~4, pp. 1367--1433, 2004.

\bibitem{DeGroot62}
M.~H. DeGroot, ``{Uncertainty, Information, and Sequential Experiments},''
  \emph{The Annals of Mathematical Statistics}, vol.~33, no.~2, pp. 404 -- 419,
  1962.

\bibitem{FarniaTse16}
F.~Farnia and D.~Tse, ``A minimax approach to supervised learning,'' in
  \emph{Conference on Neural Information Processing Systems}, 2016.

\bibitem{Kallenberg}
O.~Kallenberg, \emph{Foundations of Modern Probability}, 2nd~ed.\hskip 1em plus
  0.5em minus 0.4em\relax Springer, 2002.

\bibitem{Cover_book}
T.~Cover and J.~Thomas, \emph{Elements of Information Theory}, 2nd~ed.\hskip
  1em plus 0.5em minus 0.4em\relax New York: Wiley, 2006.

\bibitem{Good_pte}
I.~J. Good, ``{On the Principle of Total Evidence},'' \emph{The British Journal
  for the Philosophy of Science}, vol.~17, no.~4, pp. 319--321, 1967.

\bibitem{gen_learn_Huttegger}
S.~Huttegger and M.~Nielsen, ``Generalized learning and conditional
  expectation,'' \emph{Philosophy of Science (forthcoming)}, 2020.

\bibitem{Qazaz1997}
C.~S. Qazaz, C.~K.~I. Williams, and C.~M. Bishop, ``An upper bound on the
  bayesian error bars for generalized linear regression,'' in \emph{Mathematics
  of Neural Networks: Models, Algorithms and Applications}.\hskip 1em plus
  0.5em minus 0.4em\relax Boston, MA: Springer US, 1997, pp. 295--299.

\bibitem{Rissanen84}
J.~{Rissanen}, ``Universal coding, information, prediction, and estimation,''
  \emph{IEEE Transactions on Information Theory}, vol.~30, no.~4, pp. 629--636,
  1984.

\bibitem{Clarke_Barron}
B.~S. Clarke and A.~R. Barron, ``Information-theoretic asymptotics of {B}ayes
  methods,'' \emph{IEEE Trans. Inform. Theory}, vol.~36, no.~3, pp. 453--471,
  1990.

\bibitem{Cla_Bar94}
------, ``Jeffreys' prior is asymptotically least favorable under entropy
  risk,'' \emph{Journal of Statistical Planning and Inference}, vol.~41, no.~1,
  pp. 37--60, 1994.

\bibitem{Haussler_Opper_MI95}
D.~Haussler and M.~Opper, ``General bounds on the mutual information between a
  parameter and $n$ conditionally independent observations,'' in
  \emph{Proceedings of the Eighth Annual Conference on Computational Learning
  Theory}, 1995, p. 402–411.

\bibitem{BH_RPbook}
B.~Hajek, \emph{Random processes for engineers}.\hskip 1em plus 0.5em minus
  0.4em\relax Cambridge University Press, 2015.

\bibitem{WuVerdu12_fmmseMI}
Y.~{Wu} and S.~{Verd\'{u}}, ``Functional properties of minimum mean-square
  error and mutual information,'' \emph{IEEE Transactions on Information
  Theory}, vol.~58, no.~3, pp. 1289--1301, 2012.

\bibitem{Sauer72}
N.~Sauer, ``On the density of families of sets,'' \emph{Journal of
  Combinatorial Theory}, vol. 13.1, 1972.

\bibitem{Shelah72}
S.~Shelah, ``A combinatorial problem; stability and order for models and
  theories in infinitary languages,'' \emph{Pacific Journal of Mathematics},
  vol. 41.1, 1972.

\bibitem{RusZou15}
D.~Russo and J.~Zou, ``Controlling bias in adaptive data analysis using
  information theory,'' in \emph{Proceedings of The 19th International
  Conference on Artificial Intelligence and Statistics}, 2016.

\bibitem{JiaoHanWeissman_bias_17}
J.~{Jiao}, Y.~{Han}, and T.~{Weissman}, ``Dependence measures bounding the
  exploration bias for general measurements,'' in \emph{IEEE International
  Symposium on Information Theory (ISIT)}, 2017.

\bibitem{RD_MER21}
H.~Hafez-Kolahi, B.~Moniri, S.~Kasaei, and M.~S. Baghshah, ``Rate-distortion
  analysis of minimum excess risk in {B}ayesian learning,'' in
  \emph{International Conference on Machine Learning}, 2021.

\bibitem{steinke20_cmi}
T.~Steinke and L.~Zakynthinou, ``{R}easoning about generalization via
  conditional mutual information,'' in \emph{Conference on Learning Theory},
  2020.

\bibitem{Haghifam20}
M.~Haghifam, J.~Negrea, A.~Khisti, D.~M. Roy, and G.~K. Dziugaite, ``Sharpened
  generalization bounds based on conditional mutual information and an
  application to noisy, iterative algorithms,'' in \emph{Conference on Neural
  Information Processing Systems}, 2020.

\bibitem{Liu17}
J.~Liu, P.~Cuff, and S.~Verd\'u, ``On alpha-decodability and alpha-likelihood
  decoder,'' in \emph{The 55th Ann. Allerton Conf. Comm. Control Comput.},
  2017.

\bibitem{Bhatt18}
A.~Bhatt, J.-T. Huang, Y.-H. Kim, J.~J. Ryu, and P.~Sen, ``Variations on a
  theme by {Liu, Cuff, and Verd\'u}: The power of posterior sampling,'' in
  \emph{IEEE Information Theory Workshop}, 2018.

\bibitem{LieVaj06}
F.~{Liese} and I.~{Vajda}, ``On divergences and informations in statistics and
  information theory,'' \emph{IEEE Transactions on Information Theory},
  vol.~52, no.~10, pp. 4394--4412, 2006.

\bibitem{Villani_topics}
C.~Villani, \emph{Topics in Optimal Transportation}, ser. Graduate Studies in
  Mathematics.\hskip 1em plus 0.5em minus 0.4em\relax Providence, RI: Amer.
  Math. Soc., 2003, vol.~58.

\bibitem{Xu_cont_H_ISIT20}
A.~Xu, ``Continuity of generalized entropy,'' in \emph{IEEE International
  Symposium on Information Theory}, 2020.

\bibitem{Xu_cont_H_learning}
\BIBentryALTinterwordspacing
------, ``Continuity of generalized entropy and statistical learning,''
  \emph{accepted to IEEE Transactions on Information Theory}, 2020. [Online].
  Available: \url{https://arxiv.org/abs/2012.15829}
\BIBentrySTDinterwordspacing

\bibitem{PoliWu06}
Y.~Polyanskiy and Y.~Wu, ``Wasserstein continuity of entropy and outer bounds
  for interference channels,'' \emph{IEEE Transactions on Information Theory},
  vol.~62, no.~7, 2016.

\bibitem{Wu_Var}
Y.~Wu, \emph{personal communication}, 2019.

\bibitem{Talagrand96}
M.~Talagrand, ``Transportation cost for {G}aussian and other product
  measures,'' \emph{Geometric and Functional Analysis}, vol.~6, 1996.

\bibitem{fdiv_Sas_Ver}
I.~{Sason} and S.~{Verdú}, ``$f$ -divergence inequalities,'' \emph{IEEE
  Transactions on Information Theory}, vol.~62, no.~11, pp. 5973--6006, 2016.

\bibitem{math_nn21}
C.~Fang, H.~Dong, and T.~Zhang, ``Mathematical models of overparameterized
  neural networks,'' \emph{Proceedings of the IEEE}, vol. 109, 2021.

\bibitem{model_select18}
J.~Ding, V.~Tarokh, and Y.~Yang, ``Model selection techniques—an overview,''
  \emph{IEEE Signal Processing Magazine}, November 2018.

\bibitem{concen_ineq_BLM}
S.~Boucheron, G.~Lugosi, and P.~Massart, \emph{Concentration Inequalities: A
  nonasymptotic theory of independence}.\hskip 1em plus 0.5em minus 0.4em\relax
  Oxford University Press, 2013.

\bibitem{Waker_Lip_f}
H.~F. Walker, unpublished lecture notes, Worcester Polytechnic Institute.

\end{thebibliography}

\end{document}